\pdfoutput=1
\documentclass[11pt]{article}

\usepackage[T1]{fontenc}
\usepackage[utf8]{inputenc}
\usepackage{lmodern}
\usepackage[a4paper,margin=25mm]{geometry}
\usepackage{amsmath,amssymb,amsthm}
\usepackage{graphicx}
\usepackage[shortlabels]{enumitem}
\usepackage{cite}
\usepackage{microtype}
\usepackage[hyphens]{url}
\usepackage{xcolor}
\usepackage[colorlinks=true,linkcolor=blue,citecolor=blue,urlcolor=blue]{hyperref}
\usepackage{authblk}

\theoremstyle{plain}
\newtheorem{thm}{Theorem}[section]
\newtheorem*{informalthm}{Informal Theorem}
\newtheorem{lem}[thm]{Lemma}
\newtheorem{prop}[thm]{Proposition}
\newtheorem{cor}[thm]{Corollary}
\theoremstyle{definition}
\newtheorem{defn}[thm]{Definition}
\newtheorem{rmk}[thm]{Remark}

\DeclareMathOperator*{\argmin}{argmin}
\DeclareMathOperator{\Vol}{Vol}
\newcommand{\R}{\mathbb R}
\newcommand{\Sph}{\mathbb S^{d-1}}
\newcommand{\E}{\mathbb{E}}
\newcommand{\Hs}{\mathcal H_*}
\newcommand{\Rhat}{\widehat R}

\DeclareRobustCommand{\rev}[1]{{#1}}
\newcommand*{\arxiv}[1]{arXiv:\href{https://arxiv.org/abs/#1}{#1}}

\graphicspath{{}}

\numberwithin{equation}{section}

\hypersetup{
	pdftitle={Benign Overfitting of Two-Layer ReLU Networks in Fixed Dimension through Localized Initialization},
	pdfauthor={Junhyung Park, Patrick Bl\"obaum, Shiva Kasiviswanathan}}

\begin{document}

\title{Benign Overfitting of Two-Layer ReLU Networks in Fixed Dimension through Localized Initialization}

\author[1]{Junhyung Park}
\author[2]{Patrick Bl\"obaum}
\author[2]{Shiva Kasiviswanathan}
\affil[1]{Department of Mathematics, Technical University of Munich}
\affil[2]{Amazon}
\affil[ ]{\texttt{junhyung.park@tum.de}, \texttt{bloebp@amazon.com}, \texttt{kasivisw@gmail.com}}
\date{}

\maketitle

\begin{abstract}
\rev{Benign overfitting occurs when a model fits noisy training data exactly and still generalizes well.
We establish this behavior for finite-width two-layer networks with the standard ReLU activation, trained by gradient flow in fixed dimension.
Our data-independent initialization combines Gaussian weights with neurons initially active only on shrinking spherical caps.
For uniform spherical inputs and any fixed bounded conditional response law, training and excess risks vanish in probability uniformly over all sufficiently late times.
The parameters also converge, with probability tending to one, to an interpolating limit whose excess risk vanishes in probability.
We use one schedule depending only on sample size and dimension.
Our proof connects the network to a kernel flow whose limit combines ridge regression with localized spikes that fit the training residuals.
We give explicit sample-size, width and training-time bounds.}
\end{abstract}

\section{Introduction}\label{sec:introduction}

\rev{Classical risk bounds balance approximation and estimation errors, giving the familiar U-shaped risk curve \cite{gyorfi2006distribution,hastie2009elements}.
Models with many parameters can also fit noisy training data and predict accurately, a phenomenon called \emph{benign overfitting} \cite{belkin2019reconciling,bartlett2020benign,bartlett2021deep,belkin2021fit}.
In this paper, we ask whether finite neural networks can make training and excess risks both vanish as the sample size $n$ grows while the input dimension $d$ stays fixed.
We call this property \emph{almost benign overfitting}.}

\rev{We establish this property for finite-width two-layer networks with the standard ReLU activation, trained by gradient flow on uniform spherical inputs with bounded responses.
Their interpolating limits are consistent as well.
The key is an initialization independent of the data: a vanishing fraction of neurons have Gaussian weights, and the others are initially active on spherical caps around uniformly drawn centers, with radii shrinking polynomially in $n$.}

\rev{Exact interpolation in fixed dimension is delicate.
Under noisy sampling, it is inconsistent for several classes of Laplace, Sobolev, dot-product and neural tangent kernels \cite{rakhlin2019consistency,buchholz2022kernel,beaglehole2023inconsistency,li2024kernel,haas2023mind}.
In some of these classes, excess risk remains proportional to the noise level, a behavior known as \emph{tempered overfitting} \cite{mallinar2022benign,li2023asymptotic}.}

\rev{A way to fit noise while preserving generalization is to confine label corrections to small neighborhoods of the training points.
Singular-kernel interpolators can be consistent and even minimax optimal \cite{devroye1998hilbert,belkin2018overfitting,belkin2019does}.
More generally, a predictor can combine a simple estimate of the regression function with narrow spikes that fit the noise \cite{bartlett2021deep}.
Haas et al.~\cite{haas2023mind} realized this idea using a smooth kernel and a spike whose bandwidth shrinks with $n$: with suitable parameter choices, interpolation behaves like ridge regression with the smooth kernel and is consistent in fixed dimension.
They also obtained such kernels through high-frequency activation changes, using the activation-kernel correspondence of Simon et al.~\cite{simon2022reverse}.
Their consistency guarantees concern kernel interpolants; the neural-network construction uses infinite-width kernels, with finite-network training studied experimentally.
M\"ucke and Steinwart~\cite{muecke2025empirical} constructed consistent interpolating ReLU networks with two hidden layers by adding local corrections to a regression estimate.
Their algorithm constructs the networks directly; we analyze gradient-flow training from a data-independent initialization in a network with one hidden layer.}

\rev{Building on the spiky-smooth mechanism of Haas et al.~\cite{haas2023mind}, we advance its neural-network realization in three ways under our spherical-input and bounded-response assumptions.
First, we retain the standard ReLU activation and create localization through a data-independent initialization of the hidden weights and biases.
The activation remains fixed as $n$ grows.
Second, we prove a quantitative guarantee for the trained finite-width network.
For prescribed accuracy and confidence, Theorem~\ref{thm:finite} gives explicit sufficient sample sizes, widths polynomial in $n$, and training times.
Third, we control both the late training trajectory and its exact interpolating limit.
The risk bounds hold simultaneously at every sufficiently late time, and the parameters converge to an interpolant with small excess risk.
Theorem~\ref{thm:consistent} gives one initialization, width and time schedule depending only on $n$ and $d$ under which both risks vanish in probability throughout late training and at the interpolating limit.}

\rev{Our proof separates learning the regression function from fitting the training residuals.
On well-separated samples, the kernel predictor splits exactly into a global part and localized spikes (Proposition~\ref{prop:kernel-flow}).
The global part approaches ridge regression with penalty $r/n$, where $r>0$ controls regularization.
At convergence, the spikes fit the remaining residuals; their root-mean-square size on an independent input is at most $\sqrt{na^{d-1}}$, where $a$ controls the cap size.}

\rev{To pass to finite width, we analyze training as a subgradient flow \cite{Cortes2008} and compare the network with kernel flow uniformly in time.
Each local neuron affects at most one training point on separated samples.
Its movement bound is therefore smaller than a global neuron's by a factor $\sqrt n$, which improves the width bound.}

\rev{\begin{informalthm}[Theorems~\ref{thm:finite} and~\ref{thm:consistent}]
Fix $d\ge2$ and a fixed law of $(x,y)$ with $x$ uniform on the unit sphere $\Sph$ and $|y|\le1$.
We train finite-width two-layer networks with the standard ReLU activation by gradient flow, using $n$ i.i.d.\ observations and a data-independent localized initialization.
For any $\epsilon,\delta\in(0,1)$, explicit sufficient sample-size, width and training-time bounds ensure, with probability at least $1-\delta$, that every solution satisfies
\[
\sup_{t\ge T}\bigl\{\max\{\text{training risk at }t,\ \text{excess risk at }t\}\bigr\}
\le\epsilon
\]
and its parameters converge as $t\to\infty$ to an interpolating limit with excess risk at most $\epsilon$.
A single initialization, width and time schedule depending only on $n$ and $d$ makes both risks vanish in probability as $n\to\infty$, uniformly over all sufficiently late times and at the interpolating limit.
\end{informalthm}}

\rev{Theorem~\ref{thm:finite} gives an explicit sufficient sample size and a width polynomial in $n$.
Corollary~\ref{cor:log-confidence} refines the sample-size bound to depend logarithmically on the inverse failure probability.
Theorem~\ref{thm:consistent} gives one schedule depending only on $n$ and $d$ under which both risks vanish in probability, uniformly over sufficiently late training times and at the interpolating limit.
For regression functions in the global kernel's reproducing kernel Hilbert space (RKHS), both risks are of order at most $n^{-1/2}$ in probability (Remark~\ref{rmk:rates}).
The raw training-time scale grows as $n^3$, up to logarithmic factors; see~\eqref{eq:scales}.}

\rev{Figure~\ref{fig:main} (Section~\ref{sec:experiments}) illustrates the underlying kernel mechanism: training risk approaches zero along the flow, while late-time excess risk decreases across the tested sample sizes.}

\subsection{Related Work}\label{subsec:related}

\rev{Benign overfitting has been proved in linear regression \cite{bartlett2020benign,muthukumar2020harmless,zou2021benign,koehler2021uniform,tsigler2023benign}, in kernel and random-feature regression \cite{liang2020just,liang2020multiple,ghorbani2021linearized,mei2022generalization,hastie2022surprises,montanari2022interpolation,cheng2024characterizing,zhou2024agnostic}, and for neural network classifiers \cite{frei2022benign,cao2022benign,frei2023benign,kou2023benign,xu2025rethinking}.
The kernel and random-feature results mostly concern high dimension.
Ko and Huo~\cite{ko2024universal} study consistency and interpolation for ReLU classifiers defined by logistic-loss empirical minimization; our results concern squared-loss regression along gradient flow and at its interpolating limit.
Tempered overfitting has also been shown for networks \cite{kornowski2023tempered,harel2024provable}.}

\rev{Other inconsistency results cover Gaussian ridgeless regression with varying bandwidth \cite{medvedev2024overfitting} and interpolation measured in Sobolev norms \cite{yang2025sobolev}.
Related obstructions hold for near-minimum-norm interpolation in Sobolev spaces \cite{karhadkar2026harmful} and minimum-norm univariate ReLU networks \cite{joshi2024noisy}.
In linear inverse problems, smoothing by the forward operator can make kernel interpolation benign in fixed dimension \cite{wong2024benign}.}

\rev{Under independent additive noise and suitable schedules, the bounds of Barzilai and Shamir~\cite[Theorem~4.1]{barzilai2024generalization} also imply consistency for global-local kernel mixtures like ours.
The thresholded ReLU used by our local neurons appears, up to scaling, in the lower-bound constructions of Liang et al.~\cite{liang2025shattering}.}

\rev{Early stopping controls the trade-off between approximation and estimation, and spectral-filter arguments relate gradient flow to ridge regression \cite{yao2007early,raskutti2014early,ali2019continuous}.
Allerbo~\cite{allerbo2023changing} decreases the bandwidth during training; our kernel is fixed along each trajectory.
For networks, early stopping gives robustness to label noise and consistency \cite{li2020gradient,ji2021early,kuzborskij2022learning}.
Overparameterized deep networks trained by gradient descent from suitable random initialization are universally consistent \cite{drews2024universal}, although empirical-risk minimization alone need not generalize \cite{kohler2021over}.
Yang and Li~\cite{yang2025sharp} give finite-width bounds for ReLU networks trained by gradient descent with trained biases and early stopping, for targets in a tangent-kernel RKHS.
Our guarantees hold throughout late training and at the interpolating limit.}

\rev{Other studies examine how risks change with sample size \cite{nakkiran2021deep}, the ratio of parameters to samples \cite{wen2023benign}, and the trade-offs between model size, training loss and test loss \cite{ghosh2023universal}.}

\rev{Our analysis uses the kernel (lazy) regime \cite{jacot2018neural,chizat2019lazy}.
We build on optimization guarantees \cite{du2019iclr,allen2019convergence,oymak2020toward,nguyen2021proof,razborov2022improved} and generalization analyses \cite{arora2019fine,montanari2022interpolation,bowman2022spectral} for two-layer networks, and on uniform-in-time comparisons between networks and their kernels \cite{lai2023generalization}.
Chen et al.~\cite{chen2024initialization} study such comparisons together with the statistical effect of nonzero random initial outputs.}

\section{Setup}\label{sec:setup}

Throughout, the dimension $d\ge2$ and the distribution of $(x,y)$ are fixed.
The input $x$ is uniform on the sphere $\Sph$, with law $\rho$, and the response satisfies $|y|\le1$.
The regression function $f^\star(x)=\E[y\mid x]$ then satisfies $|f^\star|\le1$ but is otherwise an arbitrary measurable function.
The noise $\xi=y-f^\star(x)$ has $\E[\xi\mid x]=0$ and variance $\sigma^2=\E[\xi^2]\le1$, and its conditional variance may depend on $x$.
The training sample $S_n=((x_i,y_i))_{i=1}^n$ is i.i.d.; we write $X=(x_1,\ldots,x_n)$, $y=(y_1,\ldots,y_n)^\top$ and $f(X)=(f(x_1),\ldots,f(x_n))^\top$.
For $u\in\R^n$ let $\lVert u\rVert_n=\lVert u\rVert_2/\sqrt n$, and let $\lVert g\rVert_\rho$ be the norm of $L^2(\rho)$.
We use $\lVert\cdot\rVert_{\mathrm{op}}$ for operator norms, $\lVert\cdot\rVert_{\mathrm F}$ for the Frobenius norm, and $A\succeq B$ when $A-B$ is positive semidefinite.
\rev{We use $O_d$ for upper bounds and $\asymp_d$ for two-sided bounds up to positive constants depending only on $d$.
For positive $b_n$, $Z_n=O_{\Pr}(b_n)$ means that $Z_n/b_n$ is bounded in probability.}

For the square loss, the population risk of $f:\Sph\to\R$ is $R(f)=\E[(f(x)-y)^2]=\sigma^2+\lVert f-f^\star\rVert_\rho^2$, so the \emph{excess risk} is $R(f)-R(f^\star)=\lVert f-f^\star\rVert_\rho^2$; for a fitted function, $R$ is evaluated at an independent test point.
The \emph{training risk} is $\Rhat(f)=\lVert f(X)-y\rVert_n^2$.

\begin{defn}[\rev{Benign and Almost Benign Overfitting}]\label{def:benign}
For each $n$, let $\widehat f_n$ be the output of a learning rule applied to $S_n$ and to independent randomness, such as a random initialization.
The sequence achieves \emph{almost benign overfitting} if for all $\epsilon,\delta\in(0,1)$ there is an $N$ such that for all $n\ge N$,
\[
\Pr\bigl[\Rhat(\widehat f_n)\le\epsilon\ \text{and}\ R(\widehat f_n)-R(f^\star)\le\epsilon\bigr]\ge1-\delta .
\]
It achieves \emph{benign overfitting} if, in addition, $\widehat f_n(x_i)=y_i$ for all $i$ with probability tending to one.
\end{defn}

Almost benign overfitting thus asks that both risks tend to zero in probability.
It relaxes benign overfitting by asking for a small rather than a zero training risk, which still forces the predictor to fit the noise when $\sigma^2>0$.
\rev{This definition applies to finite-time predictors along a trajectory converging to interpolation.}

\section{\rev{Noise Fitting and Norm Constraints}}\label{sec:standard}

\rev{We first show that, with nonzero noise, almost benign overfitting separates empirical and population norms of the estimation error.
We then bound how much noise a predictor can fit at a given RKHS norm.}
For $u,w\in\R^n$ we write $\langle u,w\rangle_n=u^\top w/n$, and $\xi=(\xi_1,\ldots,\xi_n)^\top$.

\begin{lem}\label{lem:localization}
Suppose that $(\widehat f_n)$ achieves almost benign overfitting, and let $e_n=\widehat f_n-f^\star$.
Then, in probability,
\[
\lVert e_n(X)\rVert_n\to\sigma,\qquad \langle e_n(X),\xi\rangle_n\to\sigma^2,\qquad \lVert e_n\rVert_\rho\to0 .
\]
\end{lem}

\begin{proof}
Since the $\xi_i$ are i.i.d.\ and bounded, $\lVert\xi\rVert_n^2\to\sigma^2$ almost surely.
The training risk is $\Rhat(\widehat f_n)=\lVert e_n(X)-\xi\rVert_n^2$, which tends to zero in probability.
The triangle inequality gives $\bigl|\lVert e_n(X)\rVert_n-\lVert\xi\rVert_n\bigr|\le\Rhat(\widehat f_n)^{1/2}$, which proves the first claim, and the identity $2\langle e_n(X),\xi\rangle_n=\lVert e_n(X)\rVert_n^2+\lVert\xi\rVert_n^2-\Rhat(\widehat f_n)$ proves the second.
The third claim is the definition of almost benign overfitting.
\end{proof}

\rev{The estimation error reproduces the noise at the sample points while its population norm vanishes.
For every fixed $b>0$, the set $\{x\mid |e_n(x)|>b\}$ has probability at most $\lVert e_n\rVert_\rho^2/b^2$, which tends to zero in probability.
For $\sigma>0$, this rules out error classes on which empirical norms are uniformly bounded by a fixed multiple of population norms with probability tending to one.}
Localized corrections near the training points are one way to separate the two norms.

The next result, proved in Appendix~\ref{app:standard}, bounds how much noise a predictor of controlled RKHS norm can fit.

\begin{prop}\label{prop:fixed-complexity}
Let $k$ be a continuous positive semidefinite kernel on $\Sph$, with RKHS $\mathcal H_k$ and $\kappa=\sup_x\{k(x,x)\}<\infty$.
For $B\ge0$, every measurable, possibly randomized estimator $\widehat f$ taking values in $\mathcal H_k$ and satisfying $\E[\lVert\widehat f\rVert_{\mathcal H_k}^2]\le B^2$ obeys
\[
\E[\Rhat(\widehat f)]\ge\sigma^2-2B\sigma\sqrt{\kappa/n}.
\]
In particular, let $(\widehat f_t)_{t\ge0}$ be kernel gradient flow on the training risk, started from zero.
For any nonnegative, measurably chosen time $\widehat t$ with $\E[\widehat t]<\infty$,
\[
\E[\Rhat(\widehat f_{\widehat t})]
\ge\sigma^2-2\sigma\sqrt{\frac{\kappa\E[\widehat t]}{n}}.
\]
\end{prop}

For fixed $\sigma>0$ and $\kappa>0$, the first bound says that a predictor with vanishing expected training risk has root-mean-square RKHS norm of order at least $\sqrt n$.
\rev{The second bound allows any measurable time choice that depends on the observations.
Along kernel flow, $0\le\Rhat(\widehat f_t)\le\lVert y\rVert_n^2\le1$.
If $\Rhat(\widehat f_{\widehat t})\le\epsilon$ with probability at least $1-\delta$, then $\E[\Rhat(\widehat f_{\widehat t})]\le\epsilon+\delta$.
For $\epsilon+\delta<\sigma^2$ and $\kappa>0$, we obtain}
\[
\E[\widehat t]\ge
\frac{n(\sigma^2-\epsilon-\delta)^2}{4\kappa\sigma^2}.
\]
With a fixed kernel, fitting a fixed fraction of the noise therefore takes an expected training time of order at least $n$.
\rev{Kernel flow interpolates as $t\to\infty$ when its Gram matrix is invertible.
The fixed-dimensional inconsistency results cited in the introduction concern this interpolation regime.}
Section~\ref{sec:localized} therefore uses a kernel with two parts: a fixed global part that estimates $f^\star$, and a local part, shrinking with $n$, that fits the remaining noise.

\section{Localization from Initialization}\label{sec:localized}

\rev{We first specify the network, gradient-flow training and localized initialization in Section~\ref{subsec:model}.
In Section~\ref{subsec:kernel-flow}, we decompose kernel flow into global and local components and relate its limit to ridge regression.
We then state the finite-width risk bounds and consistency guarantees, including those at the interpolating limit, in Section~\ref{subsec:main}.
Finally, Section~\ref{subsec:outline} explains the subclaims and estimates used to prove these results.}

\subsection{Network, Training and Initialization}\label{subsec:model}

Let $\phi(u)=\max\{u,0\}$ and $z_x=(x,1)\in\R^{d+1}$.
The network of even width $m$ is
\[
f_W(x)=\frac1{\sqrt m}\sum_{j=1}^m s_j\phi(\theta_j^\top z_x),
\]
where the rows $\theta_j=(w_j,b_j)$ of $W\in\R^{m\times(d+1)}$ contain the hidden weights and biases, and the output signs $s_j\in\{-1,1\}$ are fixed, as is common in theoretical analyses \cite{bartlett2021deep,montanari2022interpolation}.
We train all hidden weights and biases by gradient flow on the training risk, $\dot W(t)\in-\partial^C\Rhat(f_{W(t)})$, where $\partial^C$ is the Clarke subdifferential with respect to $W$; at points of differentiability this is the usual gradient flow.
Global solutions exist, and all our guarantees hold simultaneously for every solution (Appendix~\ref{app:flow}).

We use paired initializations: the rows $\theta_j(0)$, $j\le m/2$, are drawn independently of one another and of the observations, and $\theta_{j+m/2}(0)=\theta_j(0)$, $s_j=1$, $s_{j+m/2}=-1$, so that $f_{W(0)}\equiv0$ \cite{zhang2020type,bowman2022spectral,montanari2022interpolation}.

The \emph{localized initialization} has two parameters: $r>0$, which will act as an absolute ridge penalty (Proposition~\ref{prop:kernel-flow}), and a radius $0<a<1/2$.
Let $h=a^2/8$, let $C\sim\rho$ and $Z\sim N(0,I_d)$, and set
\begin{equation}\label{eq:mixture}
v=\Pr[C^\top x>1-h],\qquad p=\frac{2v}{r+2v},\qquad \nu=pr,
\end{equation}
where $v$ does not depend on $x\in\Sph$.
Each representative row is, independently,
\begin{equation}\label{eq:init}
\theta_j(0)=\begin{cases}(Z_j,-1)&\text{with probability }p\quad\text{(global neuron)},\\ (C_j/(1-h),-1)&\text{with probability }1-p\quad\text{(local neuron)},\end{cases}
\end{equation}
with $Z_j\sim N(0,I_d)$ and $C_j\sim\rho$.
A local neuron is initially active only on the spherical cap $\{x\mid \lVert x-C_j\rVert<a/2\}$, where its activation $(1-h)^{-1}\phi(C_j^\top x-(1-h))$ is, up to scaling, the cap atom of Liang et al.~\cite[Appendix~H.1]{liang2025shattering}.
A global neuron is active on the cap $\{x\mid Z_j^\top x>1\}$, which is empty when $\lVert Z_j\rVert\le1$ and otherwise has angular radius below $\pi/2$, possibly very small.
\rev{Every row has $\E[\lVert\theta_j(0)\rVert^2]\le d+1$.}

The expected tangent kernel of the localized initialization is $k=pk_*+(1-p)k_h$, where
\begin{equation}\label{eq:kernels}
\begin{aligned}
k_*(x,u)&=(1+x^\top u)\Pr[Z^\top x>1,\ Z^\top u>1],\\
k_h(x,u)&=(1+x^\top u)\Pr[C^\top x>1-h,\ C^\top u>1-h],
\end{aligned}
\end{equation}
and the factor $1+x^\top u=z_x^\top z_u$ comes from training the biases.
Let $\Hs$ be the RKHS of the \emph{global kernel} $k_*$, let $G=(k_*(x_i,x_j))_{i,j}$ and $K=(k(x_i,x_j))_{i,j}$, and write $q=d-1$.

\rev{The following lemma shows that the global kernel can approximate the target and the local kernel acts independently at separated training points; see Appendix~\ref{app:kernels}.}
\rev{A function orthogonal to global kernel sections has zero cap averages and hence vanishes; geometry and a union bound give localization and separation.}

\begin{lem}\label{lem:kernels}
The global and local kernels have the following properties.
\begin{enumerate}[(i)]
\item The kernel $k_*$ is continuous, has constant diagonal $a_*=2\Pr[N(0,1)>1]\approx0.317$, and $\Hs$ is dense in $L^2(\rho)$.
\item The local kernel satisfies $k_h(x,x)=2v$, and $k_h(x,u)=0$ whenever $\lVert x-u\rVert\ge a$.
\item On the event $E_{\mathrm{sep}}=\{\min_{i\ne j}\{\lVert x_i-x_j\rVert\}>2a\}$, we have $K=p(G+rI_n)\succeq\nu I_n$. Moreover, $\Pr[E_{\mathrm{sep}}^c]\le n^2(2a)^q/2$.
\end{enumerate}
\end{lem}

\subsection{Kernel-Flow Decomposition}\label{subsec:kernel-flow}

As the width grows, the network dynamics approach kernel flow with the expected tangent kernel $k$.
Started from zero, kernel flow has training residual and predictor
\[
y-f^k_t(X)=e^{-2tK/n}y,\qquad f^k_t(x)=\frac2n\int_0^tk(x,X)^\top e^{-2sK/n}y\,ds .
\]
Let $\widehat f_{*,\lambda}=\argmin_{f\in\Hs}\{\Rhat(f)+\lambda\lVert f\rVert_{\Hs}^2\}$ be kernel ridge regression with the global kernel.
We use the normalized time $\tau=2\nu t/n$.

\rev{Kernel flow decomposes into a part approaching ridge regression and local corrections with a controlled population norm.}
\rev{Substituting $K=p(G+rI_n)$ into the explicit flow gives the decomposition; spectral decay, a resolvent bound and disjoint cap supports control its components.}

\begin{prop}\label{prop:kernel-flow}
On $E_{\mathrm{sep}}$, for every $t\ge0$,
\[
f^k_t=\underbrace{k_*(\cdot,X)^\top\alpha_t}_{\text{global part}}+\underbrace{\frac{1-p}pk_h(\cdot,X)^\top\alpha_t}_{\text{spikes}},\qquad \alpha_t=(G+rI_n)^{-1}\bigl(I_n-e^{-2tK/n}\bigr)y .
\]
With $\lambda=r/n$, the three components satisfy
\begin{gather*}
\lVert y-f^k_t(X)\rVert_n\le e^{-\tau},\qquad \bigl\lVert k_*(\cdot,X)^\top\alpha_t-\widehat f_{*,\lambda}\bigr\rVert_\rho\le\frac{e^{-\tau}}{2\sqrt\lambda},\\
\Bigl\lVert\frac{1-p}pk_h(\cdot,X)^\top\alpha_t\Bigr\rVert_\rho\le\sqrt{na^q}.
\end{gather*}
\end{prop}

The global part converges to ridge regression in the fixed RKHS $\Hs$ with penalty $\lambda=r/n$, a classical estimator of $f^\star$.
The spikes add $r\alpha_t$ to the predictions at the sample points, and in the limit they fit the residual left by ridge regression, so that $f^k_t$ interpolates as $t\to\infty$.
Their population norm is at most $\sqrt{na^q}$, which is small when the radius $a$ is much smaller than the sample spacing $n^{-1/q}$.
\rev{Figure~\ref{fig:main}(a) illustrates these trajectories with $r_n=n^{1/4}$; Appendix~\ref{app:kernels} contains the proof.}
Normalized time $\tau$ corresponds to raw time $t=\tau H/2$, where $H=n/\nu$; under the schedules below, $r\ge4$ and $\nu=2(1-p)v\asymp_d a^q$, so smaller caps mean slower training.

\subsection{Main Results}\label{subsec:main}

For accuracy $\epsilon\in(0,1)$ and confidence $\delta\in(0,1)$, we choose the schedule
\begin{equation}\label{eq:schedule}
r=\frac{64}{\epsilon\delta},\qquad \lambda=\frac rn,\qquad a=\frac12\Bigl(\frac{\delta}{3n^2}\Bigr)^{1/q},\qquad H=\frac n\nu,\qquad T=\frac H2\log\Bigl(\frac4{\sqrt{\epsilon\lambda}}\Bigr) .
\end{equation}
Here $r$ is the absolute ridge penalty, $a/2$ the radius of the local caps and $H$ the time scale of the local neurons, and after time $T$ the transient term $e^{-\tau}/(2\sqrt\lambda)$ in Proposition~\ref{prop:kernel-flow} is at most $\sqrt\epsilon/8$.
The width $m_0(n,d,r,a,\eta,\delta_0)$, defined explicitly in~\eqref{eq:width}, ensures that on $E_{\mathrm{sep}}$, with probability at least $1-\delta_0/3$ over the initialization, the network stays within $\eta$ of kernel flow at all times, both on the sample and in $L^2(\rho)$.

\rev{The next theorem bounds both risks at every sufficiently late training time and establishes parameter convergence to an accurate interpolant.}
\rev{Section~\ref{subsec:outline} outlines the proof.}

\begin{thm}\label{thm:finite}
Fix $d\ge2$, a data distribution as in Section~\ref{sec:setup}, and $\epsilon,\delta\in(0,1)$.
Let $f_\epsilon\in\Hs$ satisfy $\lVert f_\epsilon-f^\star\rVert_\rho^2\le\epsilon\delta/48$, which is possible by Lemma~\ref{lem:kernels}(i), and let $B_*=\lVert f_\epsilon\rVert_{\Hs}$.
Suppose that
\[
n\ge\max\Bigl\{\frac{64}{\epsilon\delta},\ \frac{3072B_*^2}{\epsilon^2\delta^2}\Bigr\},\qquad m\ge m_0\bigl(n,d,r,a,\tfrac{\sqrt\epsilon}4,\delta\bigr),\ m\text{ even},
\]
and initialize with~\eqref{eq:init} and the schedule~\eqref{eq:schedule}.
Then there is an event of probability at least $1-\delta$ on which every gradient-flow solution satisfies the following.
\begin{enumerate}[(i)]
\item For every $t\ge T$, we have $\Rhat(f_{W(t)})\le\epsilon/4$ and $R(f_{W(t)})-R(f^\star)\le\epsilon$.
\item $W(t)$ converges to a limit $W(\infty)$ such that $f_{W(\infty)}(x_i)=y_i$ for all $i$ and $R(f_{W(\infty)})-R(f^\star)\le49\epsilon/64$.
\end{enumerate}
The smallest admissible width satisfies
\[
m_0\le C_d\bigl[\epsilon^2n^9+\delta^{-(2+4/q)}n^{6+8/q}\bigr],
\]
where $C_d$ depends only on $d$ and is given explicitly in~\eqref{eq:width-constant}.
\end{thm}

The schedule depends only on $n,d,\epsilon,\delta$, while the sufficient sample size also depends on the target through $f_\epsilon$.
With a ridge penalty that grows logarithmically in $n/\delta$, the dependence of this sample size on $\delta$ becomes logarithmic.
\rev{The proof of the following corollary replaces the Markov bound for the ridge estimator by the concentration bound of Appendix~\ref{app:ridge-concentration}; the network comparison is unchanged.}

\begin{cor}[Logarithmic Confidence Dependence]\label{cor:log-confidence}
Fix $d$, a data distribution, and $\epsilon,\delta\in(0,1)$ as in Theorem~\ref{thm:finite}.
Choose $f_{\epsilon,\log}\in\Hs$ with $\lVert f_{\epsilon,\log}-f^\star\rVert_\rho^2\le\epsilon/32$, and put
\begin{equation}\label{eq:log-sample-size}
\begin{gathered}
B_{\log}=\lVert f_{\epsilon,\log}\rVert_{\Hs},\qquad
C_{\log}=\max\Bigl\{\frac{1024}{\epsilon},\frac{20480B_{\log}^2}{\epsilon^2}\Bigr\},\\
n\ge\Biggl\lceil2C_{\log}\log\Bigl(\frac{48C_{\log}}{\delta}\Bigr)\Biggr\rceil .
\end{gathered}
\end{equation}
Define
\begin{equation}\label{eq:log-schedule}
L_{\log}=\log\Bigl(\frac{24n}{\delta}\Bigr),\qquad
r_{\log}=\frac{512L_{\log}}{\epsilon},\qquad
\lambda_{\log}=\frac{r_{\log}}n .
\end{equation}
Use $r=r_{\log}$ and $\lambda=\lambda_{\log}$, retain the radius $a$ in~\eqref{eq:schedule}, and compute $\nu,H,T$ with these parameters.
Initialize by~\eqref{eq:init}, with even width
\[
m\ge m_0\bigl(n,d,r_{\log},a,\tfrac{\sqrt\epsilon}4,\delta\bigr).
\]
Then, on one event of probability at least $1-\delta$, every gradient-flow solution satisfies conclusions~(i) and~(ii) of Theorem~\ref{thm:finite}.
\end{cor}

The logarithmic dependence holds because the approximant and $C_{\log}$ do not depend on $\delta$.
\rev{A single schedule suffices for every admissible distribution, with both risks vanishing as the sample grows, including at the interpolating limit.}
\rev{The guarantee holds uniformly over all sufficiently late training times, including measurable times chosen from the data.}

\begin{thm}\label{thm:consistent}
Let $4\le r_n\le n$ and $\eta_n\in(0,1]$ satisfy $r_n\to\infty$, $r_n/n\to0$ and $\eta_n\to0$.
For $n\ge16$, set $\delta_n=1/\log(n+e)$, $\lambda_n=r_n/n$ and $a_n=\frac12(\delta_n/(3n^2))^{1/q}$, let $\nu_n$ be given by~\eqref{eq:mixture} with $r=r_n$ and $a=a_n$, and set $H_n=n/\nu_n$ and $T_n=\frac{H_n}2\log(n)$.
Use even widths $m_n\ge m_0(n,d,r_n,a_n,\eta_n,\delta_n)$.
For every data distribution as in Section~\ref{sec:setup}, there are events $\Omega_n$ with $\Pr[\Omega_n]\to1$ and random variables $\epsilon_n\to0$ in probability such that on $\Omega_n$ every gradient-flow solution satisfies
\[
\sup_{t\ge T_n}\{\Rhat(f_{W(t)})\}\le\epsilon_n,\qquad \sup_{t\ge T_n}\bigl\{R(f_{W(t)})-R(f^\star)\bigr\}\le\epsilon_n,
\]
and $W(t)$ converges to an interpolating limit with $R(f_{W(\infty)})-R(f^\star)\le\epsilon_n$.
In particular, any measurable choices of solutions and times $t_n\ge T_n$ achieve almost benign overfitting, and their interpolating limits, defined as in Appendix~\ref{app:flow}, achieve benign overfitting.
\end{thm}

\begin{rmk}[Rates]\label{rmk:rates}
The proof separates the error into the finite-width comparison, the kernel transient, the local correction and the global ridge error:
\[
\epsilon_n=\bigl(\eta_n+n^{-1}+\sqrt{na_n^q}+(4nr_n)^{-1/2}+\lVert\widehat f_{*,\lambda_n}-f^\star\rVert_\rho\bigr)^2,
\]
and the leave-one-out bound of Lemma~\ref{lem:ridge-oracle} gives
\[
\E[\lVert\widehat f_{*,\lambda}-f^\star\rVert_\rho^2]\le(1+r^{-1})^2\mathcal A_*(\lambda)+3\sigma^2/r,
\]
where $\mathcal A_*(\lambda)=\inf_{f\in\Hs}\{\lVert f-f^\star\rVert_\rho^2+\lambda\lVert f\rVert_{\Hs}^2\}$.
If $f^\star\in\Hs$, then $\mathcal A_*(\lambda)\le\lambda\lVert f^\star\rVert_{\Hs}^2$, and the choices $r_n=\sqrt n$ and $\eta_n=n^{-1/4}$ give training and excess risks bounded by $O_{\Pr}(n^{-1/2})$, uniformly for $t\ge T_n$ and in the limit.
For these choices a sufficient width is $O_d(n^{8}\log^3(n)+n^{6+8/q}\log^{3+4/q}(n))$.
\end{rmk}

In all three results, the cap mass, the fraction of global neurons and the time scale satisfy
\begin{equation}\label{eq:scales}
v_n\asymp_d\delta_n n^{-2},\qquad
p_n\asymp_d\frac{\delta_n}{r_n n^2},\qquad
H_n\asymp_d\frac{n^3}{\delta_n},
\end{equation}
with $\delta_n=\delta$ in Theorem~\ref{thm:finite} and Corollary~\ref{cor:log-confidence}, and $\delta_n=1/\log(n+e)$ in Theorem~\ref{thm:consistent}.
Almost all neurons are thus local, and the training time grows like $n^3$ up to logarithmic factors.

\subsection{Proof Outline}\label{subsec:outline}

\rev{We prove Theorems~\ref{thm:finite} and~\ref{thm:consistent} through four subclaims: a kernel decomposition with small localized corrections, an accurate global ridge estimator, convergence of the network parameters, and a network/kernel comparison valid at every time.
The decomposition and comparison give the following bounds on a common event:
\[
\begin{aligned}
\lVert y-f_{W(t)}(X)\rVert_n&\le\eta+e^{-\tau},\\
\lVert f_{W(t)}-f^\star\rVert_\rho
&\le\eta+\sqrt{na^q}+\frac{e^{-\tau}}{2\sqrt\lambda}
+\lVert\widehat f_{*,\lambda}-f^\star\rVert_\rho .
\end{aligned}
\]
Here $\eta$ controls the finite-width comparison, and the other three population terms are the localized correction, the kernel transient and the ridge error.
For Theorem~\ref{thm:finite}, we bound them by $\sqrt\epsilon/4$, $\sqrt\epsilon/8$, $\sqrt\epsilon/8$ and $\sqrt\epsilon/2$, respectively, for all $t\ge T$.
Steps~(1)--(4) establish these bounds; step~(5) applies them to prove consistency.}

\rev{\emph{(1) Kernel Decomposition and Localization.}
We first identify the estimator learned by the kernel flow.
The cap-volume bound makes $E_{\mathrm{sep}}$ likely; on this event, local kernel sections have disjoint supports and $K=p(G+rI_n)$ (Lemma~\ref{lem:kernels}(ii)--(iii)).
The explicit solution therefore has the same inverse $(G+rI_n)^{-1}$ as ridge regression with penalty $\lambda=r/n$.
Its global part converges to $\widehat f_{*,\lambda}$, and in the limit the spikes fit the remaining training residuals.
The coefficient bound $\lVert r\alpha_t\rVert_2\le\sqrt n$ and disjoint supports of mass at most $a^q$ give
\[
\Bigl\lVert\frac{1-p}p k_h(\cdot,X)^\top\alpha_t\Bigr\rVert_\rho^2
\le a^q\lVert r\alpha_t\rVert_2^2\le na^q.
\]
The spectral gap $K\succeq\nu I_n$ bounds the training residual by $e^{-\tau}$, and the resolvent estimate~\eqref{eq:resolvent} bounds the global transient by $e^{-\tau}/(2\sqrt\lambda)$.
These are the three estimates in Proposition~\ref{prop:kernel-flow}.}

\rev{\emph{(2) Global Ridge Accuracy.}
We next control the statistical term.
A leave-one-out comparison keeps the absolute penalty $r$ fixed when adding one observation.
The rank-one update changes that observation's residual by a factor at most $1+1/r$; empirical optimality then gives, with $c=(1+r^{-1})^2$,
\[
\E[\lVert\widehat f_{*,\lambda}-f^\star\rVert_\rho^2]
\le c\mathcal A_*(\lambda)+(c-1)\sigma^2 .
\]
The approximation term tends to zero as $\lambda=r/n\to0$, because the global RKHS is dense in $L^2(\rho)$ (Lemma~\ref{lem:kernels}(i)), while the noise term tends to zero as $r\to\infty$.
For Theorem~\ref{thm:finite}, the chosen approximant and sample size bound this expectation by $21\epsilon\delta/256$ (Lemma~\ref{lem:tuned-ridge}).
Markov's inequality then bounds the ridge error by $\sqrt\epsilon/2$ except on an event of probability at most $\delta/3$.
For Corollary~\ref{cor:log-confidence}, we obtain the same event by concentrating the empirical covariance and the centered ridge score (Appendix~\ref{app:ridge-concentration}).}

\rev{\emph{(3) Residual Decay and Parameter Convergence.}
We prove that every network trajectory stays near its initialization and converges to an interpolant.
While rows remain in the prescribed balls, a gate can change only when its initial preactivation lies in a narrow slab.
The slab probabilities and initial-kernel second moments give the certificates of Appendix~\ref{app:certificates}.
At the width~\eqref{eq:width}, Markov's inequality makes these certificates small with conditional failure probability at most $\delta/3$.
For the network's training tangent matrix $K_t$ (Lemma~\ref{lem:clarke-flow}), they give $K_t\succeq\nu I_n/2$, and hence, for $\mathbf r(t)=y-f_{W(t)}(X)$,
\[
\frac d{dt}\lVert\mathbf r(t)\rVert_n^2
\le-\frac2H\lVert\mathbf r(t)\rVert_n^2,\qquad
\lVert\mathbf r(t)\rVert_n\le e^{-t/H}.
\]
Integrating the row equations bounds global-row movement by $2\sqrt2H(1-e^{-t/H})/\sqrt m$.
Each local row can affect at most one training point, so its bound is smaller by a factor $\sqrt n$.
At every finite time these displacements are strictly inside the prescribed balls, which rules out a first exit.
The resulting integrable row velocities give parameter convergence, and the vanishing residual gives exact interpolation (Lemma~\ref{lem:residual}).}

\rev{\emph{(4) Uniform Network/Kernel Comparison.}
Finally, we transfer the kernel's population guarantee to the finite network.
Duhamel's formula expresses the residual difference through the tangent-matrix discrepancy and the integrable network residual.
For the population predictions, we integrate the semigroup before taking norms:
\[
\int_0^s e^{-2uK/n}\,du=\frac n2K^{-1}(I_n-e^{-2sK/n}).
\]
Lemma~\ref{lem:pop-operator} bounds the resulting population operator by $(2\sqrt r)^{-1}+a^{q/2}$ at every time, using the global resolvent and the disjoint local supports.
Together with the cross-kernel certificate, this yields comparison error at most $\eta$ on the sample and in $L^2(\rho)$, uniformly in time (Lemma~\ref{lem:coupling}).
The separation, ridge and initialization failures sum to at most $\delta/6+\delta/3+\delta/3$.
The four norm budgets above then prove the finite-time guarantee; parameter convergence and the vanishing transient leave limiting excess risk at most $(7\sqrt\epsilon/8)^2$.}

\rev{\emph{(5) Consistency and Interpolating Limits.}
For Theorem~\ref{thm:consistent}, we use its schedule with $r_n\to\infty$, $r_n/n\to0$ and $\eta_n\to0$.
For $t\ge T_n$, we have $\tau\ge\log(n)$, so the training residual norm is at most $\eta_n+n^{-1}$ and the population transient is at most $(4nr_n)^{-1/2}$.
The local correction vanishes because $na_n^q=\delta_n/(3\cdot2^qn)\to0$.
The ridge bound above and RKHS density from Lemma~\ref{lem:kernels}(i) give $\E[\lVert\widehat f_{*,\lambda_n}-f^\star\rVert_\rho^2]\to0$, hence consistency in probability.
The separation and initialization failures total at most $\delta_n/2\to0$.
Thus both risks vanish in probability uniformly over all $t\ge T_n$.
Parameter convergence gives exact interpolation, and the comparison passes to the limit with the transient removed, so the limiting excess risk also vanishes in probability.
Appendix~\ref{app:proofs-main} gives the complete assembly.}

\section{\rev{Localized Kernel-Flow Illustration}}\label{sec:experiments}

\rev{We numerically evaluate the localized kernel flow of Proposition~\ref{prop:kernel-flow} under the spherical data model of Section~\ref{sec:setup}.
We take $d=3$, $f^\star(x)=x_3/2$ and $y=f^\star(x)+\xi$, where $\xi=\pm1/4$ with equal probability independently of $x$.
Thus $|y|\le3/4$ and $\sigma^2=1/16$.}

\rev{For each $n\in\{250,500,1000,2000,4000\}$, we use five samples with
$r_n=n^{1/4}$, $\delta_n=1/\log(n+e)$ and $a_n=\frac12(\delta_n/(3n^2))^{1/2}$.
The cap radii $a_n/2$ range from $2.5\cdot10^{-4}$ to $1.3\cdot10^{-5}$.
We retain samples on $E_{\mathrm{sep}}$, where Proposition~\ref{prop:kernel-flow} reduces the predictor to its global component and disjoint local corrections.}

\begin{figure}[t]
\centering
\includegraphics[width=\textwidth]{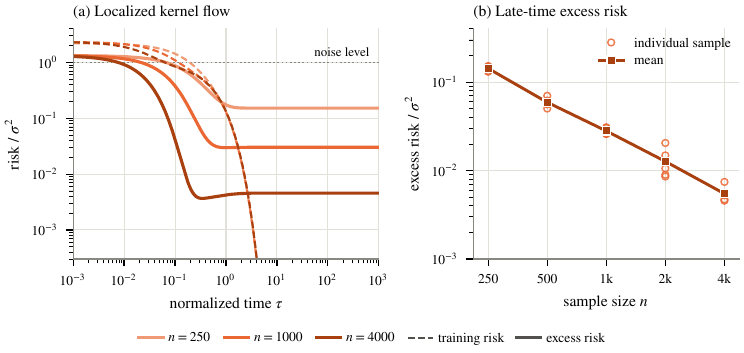}
\caption{\rev{Localized kernel-flow illustration (Section~\ref{sec:experiments}) for $d=3$, regression function $f^\star(x)=x_3/2$ and label noise $\pm1/4$, with noise variance $\sigma^2=1/16$.
We use $r_n=n^{1/4}$ and the cap schedule of Section~\ref{sec:experiments}.
(a)~Training risk (dashed) and excess risk (solid) against normalized time $\tau$ (Section~\ref{subsec:kernel-flow}), for $n=250,1000,4000$ (light to dark), one sample each.
(b)~Excess risk at $\tau=10^3$ against $n$, for all five samples per $n$ (hollow markers) and their mean.
Risks are in units of $\sigma^2$; the dotted line marks the noise level.}}
\label{fig:main}
\end{figure}

\rev{Figure~\ref{fig:main}(a) shows the risk trajectories in normalized time $\tau=2\nu t/n$.
The training risk satisfies $\Rhat(f^k_t)\le e^{-2\tau}\lVert y\rVert_n^2$.
Panel~(b) reports excess risk at $\tau=10^3$ for all five samples at each $n$; its mean decreases from $0.143\sigma^2$ to $0.0055\sigma^2$ over the tested sample sizes.
These computations illustrate the kernel mechanism: the global component estimates the regression function while the local corrections fit the training residuals.
Appendix~\ref{app:experiments} describes the numerical integration and sampling protocol.}

\section{Conclusion}\label{sec:conclusion}

Two-layer ReLU networks of finite width, trained by gradient flow, can overfit benignly in fixed dimension when their initialization contains neurons localized at a scale that shrinks with $n$.
\rev{Determining optimal sample-size, width and training-time bounds remains open.}
Other open questions are discrete-time training, input distributions other than the uniform distribution on the sphere, and whether localized features can develop during training from an initialization that does not depend on $n$.

\appendix

\section{Proof of Proposition~\ref{prop:fixed-complexity}}\label{app:standard}

\rev{We bound noise fitting by the RKHS norm and derive the corresponding constraint on kernel-flow training time.}

Write $k_i=k(\cdot,x_i)$ and consider the RKHS-valued random sum
\[
U_n=\sum_{i=1}^n\xi_i k_i.
\]
Since $k$ is continuous on the compact sphere, $\mathcal H_k$ is separable and $x\mapsto k(\cdot,x)$ is continuous, so $U_n$ is a measurable random element of $\mathcal H_k$ with finite second moment.
The summands are independent with $\E[\xi_i k_i]=0$, so the cross terms vanish and
\[
\E[\lVert U_n\rVert_{\mathcal H_k}^2]
=\sum_{i=1}^n\E[\xi_i^2 k(x_i,x_i)]
\le n\kappa\sigma^2.
\]

For any fitted function, expanding the training risk gives
\begin{align*}
\Rhat(\widehat f)
&=\lVert\widehat f(X)-f^\star(X)\rVert_n^2
  +\lVert\xi\rVert_n^2
  -\frac2n\sum_i\xi_i\widehat f(x_i)
  +\frac2n\sum_i\xi_i f^\star(x_i).
\end{align*}
The first term is nonnegative, the second has expectation $\sigma^2$, and the last has expectation zero because $\E[\xi_i\mid x_i]=0$.
By the reproducing property, the third term is $-2\langle\widehat f,U_n\rangle_{\mathcal H_k}/n$.
Taking expectations and applying Cauchy--Schwarz in the RKHS and then in probability gives
\begin{align*}
\E[\Rhat(\widehat f)]
&\ge\sigma^2-\frac2n\E[\langle\widehat f,U_n\rangle_{\mathcal H_k}]\\
&\ge\sigma^2-\frac2n
\bigl(\E[\lVert\widehat f\rVert_{\mathcal H_k}^2]\bigr)^{1/2}
\bigl(\E[\lVert U_n\rVert_{\mathcal H_k}^2]\bigr)^{1/2}\\
&\ge\sigma^2-2B\sigma\sqrt{\kappa/n}.
\end{align*}

For kernel gradient flow, let $r_i(t)=y_i-\widehat f_t(x_i)$ and $L(t)=\Rhat(\widehat f_t)$.
The path lies in the span of $k_1,\ldots,k_n$, and its defining equation and energy identity are
\[
\dot{\widehat f}_t=\frac2n\sum_i r_i(t)k_i,
\qquad
L'(t)=-\frac2n\sum_i r_i(t)\dot{\widehat f}_t(x_i)
      =-\lVert\dot{\widehat f}_t\rVert_{\mathcal H_k}^2.
\]
Since $\widehat f_0=0$, Cauchy--Schwarz for its time integral gives, simultaneously for all $t\ge0$,
\[
\lVert\widehat f_t\rVert_{\mathcal H_k}^2
\le t\int_0^t\lVert\dot{\widehat f}_s\rVert_{\mathcal H_k}^2\,ds
=t[L(0)-L(t)]
\le t\lVert y\rVert_n^2
\le t.
\]
The path is jointly measurable in the sample and time, since the coefficients of its kernel expansion are continuous matrix functions of the Gram matrix and time.
Evaluating the pathwise bound at the measurable time $\widehat t$ and taking expectations gives $\E[\lVert\widehat f_{\widehat t}\rVert_{\mathcal H_k}^2]\le\E[\widehat t]$, and the first part of the proposition with $B^2=\E[\widehat t]$ proves the claim.

\section{Kernel Lemmas}\label{app:kernels}

This appendix proves Lemma~\ref{lem:kernels} and Proposition~\ref{prop:kernel-flow}, and the ridge-regression bounds used in Section~\ref{sec:localized}.

For a measurable positive semidefinite kernel $\ell$ on $\Sph$ with $\sup_x\{\ell(x,x)\}\le\kappa_\ell$ and separable RKHS $\mathcal H_\ell$, every $f\in\mathcal H_\ell$ satisfies $|f(x)|\le\sqrt{\kappa_\ell}\lVert f\rVert_{\mathcal H_\ell}$, so $\lVert f\rVert_\rho\le\sqrt{\kappa_\ell}\lVert f\rVert_{\mathcal H_\ell}$.
The unscaled sampling operator $\mathcal S_\ell f=f(X)$, with the Euclidean norm on $\R^n$, has adjoint $\mathcal S_\ell^*u=\sum_iu_i\ell(\cdot,x_i)$, and $\mathcal S_\ell\mathcal S_\ell^*=L=(\ell(x_i,x_j))_{i,j}$.
For $s>0$, the operator $\mathcal S_\ell^*(L+sI_n)^{-1}$ has squared norm $\lVert(L+sI_n)^{-1}L(L+sI_n)^{-1}\rVert_{\mathrm{op}}\le\sup_{u\ge0}\{u/(u+s)^2\}=1/(4s)$, so
\begin{equation}\label{eq:resolvent}
\lVert\mathcal S_\ell^*(L+sI_n)^{-1}\rVert_{\mathrm{op}}\le\frac1{2\sqrt s}.
\end{equation}

\subsection{Spherical Caps and the Local Kernel}

\rev{We begin with bounds on the probability mass of small spherical caps to control sample separation and the population effect of local corrections.}

\begin{lem}\label{lem:caps}
Let $q=d-1$.
For every $u\in\Sph$ and $0<b\le1$, $\rho\{x\mid \lVert x-u\rVert\le b\}\le b^q$.
For $C\sim\rho$, the coordinate $C^\top u$ has density $c_d(1-t^2)^{\omega_d}$ on $(-1,1)$, where $c_d=\Gamma(d/2)/(\sqrt\pi\Gamma((d-1)/2))$ and $\omega_d=(d-3)/2$.
The cap mass $v=\Pr[C^\top u>1-h]$ satisfies $c_{\mathrm{cap}}h^{q/2}\le v\le(2h)^{q/2}$ for $0<h\le1/2$, with $c_{\mathrm{cap}}=(2c_d/q)\min\{1,2^{\omega_d}\}$.
\end{lem}

\begin{proof}
We rotate $u$ to the north pole \rev{and write $B_q$ for the unit ball in $\R^q$}.
The chord ball of radius $b\le1$ has last coordinate at least $1-b^2/2\ge1/2$, and it projects into the Euclidean $q$-ball of radius $b$ in the first $q$ coordinates.
The graph area Jacobian on this part of the sphere is at most $2$, so the ball has area at most $2b^q\Vol(B_q)$, while each hemisphere has graph area Jacobian at least one over $B_q$, so the sphere has area at least $2\Vol(B_q)$.
This proves the first claim.
Slicing the sphere at last coordinate $t$ gives a density proportional to $(1-t^2)^{(d-3)/2}$, with normalization $\int_{-1}^1(1-t^2)^{\omega_d}dt=\sqrt\pi\Gamma((d-1)/2)/\Gamma(d/2)$.
With $s=1-t$, $v=c_d\int_0^h[s(2-s)]^{\omega_d}ds\ge c_d\min\{1,2^{\omega_d}\}\int_0^hs^{\omega_d}ds=c_{\mathrm{cap}}h^{q/2}$.
The cap is a chord ball of radius $\sqrt{2h}\le1$, which gives the upper bound.
\end{proof}

\begin{proof}[Proof of Lemma~\ref{lem:kernels}(ii) and (iii)]
The feature $\mathbf{1}\{C^\top x>1-h\}$ is active exactly when $\lVert x-C\rVert^2=2-2C^\top x<2h=a^2/4$.
If it is active at $x$ and at $u$, then $\lVert x-u\rVert<a$.
Hence $k_h(x,u)=0$ when $\lVert x-u\rVert\ge a$, and $k_h(x,x)=2v$.
Moreover $0\le k_h(x,u)\le2v$.
The kernel $k_h$ is continuous by dominated convergence, since cap boundaries are $\rho$-null.
On $E_{\mathrm{sep}}$, the matrix $(k_h(x_i,x_j))_{i,j}$ is therefore $2vI_n$, and $K=pG+2(1-p)vI_n=p(G+rI_n)\succeq\nu I_n$, since $2(1-p)v=pr=\nu$.
Since $2a<1$, Lemma~\ref{lem:caps} and a union bound over pairs give $\Pr[E_{\mathrm{sep}}^c]\le\frac{n(n-1)}2(2a)^q$.
\end{proof}

On $E_{\mathrm{sep}}$ the balls $\{x\mid \lVert x-x_i\rVert<a\}$ are disjoint and have mass at most $a^q$, and $\frac{1-p}p k_h(\cdot,x_i)$ is supported in the $i$-th ball and bounded by $\frac{2(1-p)v}p=r$.
Hence, for every $\alpha\in\R^n$,
\begin{equation}\label{eq:spike-norm}
\Bigl\lVert\frac{1-p}p\sum_i\alpha_ik_h(\cdot,x_i)\Bigr\rVert_\rho^2=\sum_i\alpha_i^2\Bigl\lVert\frac{1-p}pk_h(\cdot,x_i)\Bigr\rVert_\rho^2\le r^2a^q\lVert\alpha\rVert_2^2 .
\end{equation}

\subsection{Global Kernel}

\begin{proof}[Proof of Lemma~\ref{lem:kernels}(i)]
Let $\chi_Z(x)=\mathbf{1}\{Z^\top x>1\}$.
Then
\[
k_*(x,u)=\E[z_x^\top z_u\chi_Z(x)\chi_Z(u)],
\]
so $k_*$ is positive semidefinite, and $Z^\top x\sim N(0,1)$ gives the diagonal $a_*$.
For fixed $x$, $\Pr[Z^\top x=1]=0$, so dominated convergence shows that $k_*$ is continuous; hence $\Hs$ is separable and consists of continuous functions.

Suppose that $w\in L^2(\rho)$ is orthogonal to every section $k_*(\cdot,u)$.
Then, by Fubini,
\[
0=\iint w(x)k_*(x,u)w(u)\,d\rho(x)\,d\rho(u)=\E_Z\Bigl[\Bigl\lVert\int w(x)z_x\chi_Z(x)\,d\rho(x)\Bigr\rVert_2^2\Bigr] .
\]
For every fixed $Z$, the set $\{x\mid Z^\top x=1\}$ is empty, a singleton, or a lower-dimensional sphere section, and is therefore $\rho$-null for $d\ge2$.
The vector integral is consequently continuous in $Z$ by dominated convergence with envelope $\sqrt2|w|$.
Since the Gaussian law has full support, the integral vanishes for every $Z$.
Its last coordinate at $Z=u/t$, with $u\in\Sph$ and $0<t<1$, gives $\int_{\{u^\top x>t\}}w\,d\rho=0$, so $w$ integrates to zero over every chord ball of radius $b<1$.
Set $v_b=\rho\{x\mid \lVert x-u\rVert<b\}$, positive and independent of $u$ by rotation invariance, and $A_bg(u)=v_b^{-1}\int_{\{\lVert x-u\rVert<b\}}g(x)\,d\rho(x)$.
Then $A_bw=0$; moreover $\lVert A_bg\rVert_\rho\le\lVert g\rVert_\rho$ by Jensen's inequality and Fubini, and $A_bg\to g$ uniformly for continuous $g$ as $b\downarrow0$.
Since continuous functions are dense in $L^2(\rho)$, $A_bw\to w$ in $L^2(\rho)$, so $w=0$.
Hence the span of the sections is dense in $L^2(\rho)$, and so is $\Hs$.
\end{proof}

For $\gamma>0$ define the approximation functional $\mathcal A_*(\gamma)=\inf_{f\in\Hs}\{\lVert f-f^\star\rVert_\rho^2+\gamma\lVert f\rVert_{\Hs}^2\}$.
Density implies $\mathcal A_*(\gamma)\to0$ as $\gamma\downarrow0$: given $\varepsilon_0>0$, choose $f\in\Hs$ with $\lVert f-f^\star\rVert_\rho^2<\varepsilon_0/2$, and then take $\gamma$ small enough that $\gamma\lVert f\rVert_{\Hs}^2<\varepsilon_0/2$.

\subsection{Kernel Flow}

\begin{proof}[Proof of Proposition~\ref{prop:kernel-flow}]
On $E_{\mathrm{sep}}$, $K\succeq\nu I_n$ is invertible, so
\[
\lVert y-f^k_t(X)\rVert_n=\lVert e^{-2tK/n}y\rVert_n\le e^{-2\nu t/n}\lVert y\rVert_n\le e^{-\tau},
\]
and integrating the definition gives $f^k_t=k(\cdot,X)^\top K^{-1}(I_n-e^{-2tK/n})y$.
Substituting $k=pk_*+(1-p)k_h$ and $K^{-1}=p^{-1}(G+rI_n)^{-1}$ gives the decomposition.
Since $G+rI_n$ commutes with $K$, $\lVert\alpha_t\rVert_2\le\lVert y\rVert_2/r\le\sqrt n/r$, and~\eqref{eq:spike-norm} bounds the spikes by $\sqrt{na^q}$.
Ridge regression with $k_*$ and $\lambda=r/n$ is $\widehat f_{*,\lambda}=k_*(\cdot,X)^\top(G+rI_n)^{-1}y$, so the global part differs from it by $-\mathcal S_*^*(G+rI_n)^{-1}e^{-2tK/n}y$, where $\mathcal S_*$ is the sampling operator of $k_*$.
By~\eqref{eq:resolvent} and $\lVert f\rVert_\rho\le\lVert f\rVert_{\Hs}$, its population norm is at most $\sqrt ne^{-\tau}/(2\sqrt r)=e^{-\tau}/(2\sqrt\lambda)$.
\end{proof}

\rev{For the finite-width comparison, the next lemma bounds how perturbations at training points affect population predictions, uniformly over time.}
Let $\mathcal S$ be the sampling operator of $k$ and $C_{\mathrm{pop}}=(2\sqrt\lambda)^{-1}+\sqrt{na^q}$.

\begin{lem}\label{lem:pop-operator}
On $E_{\mathrm{sep}}$, for every $t\ge0$, the map $u\mapsto\mathcal S^*K^{-1}(I_n-e^{-2tK/n})u$ from $\R^n$ to $L^2(\rho)$ has norm at most $C_{\mathrm{pop}}/\sqrt n$.
\end{lem}

\begin{proof}
Let $Q_t=I_n-e^{-2tK/n}$, a contraction that commutes with $G$.
As in the previous proof, $\mathcal S^*K^{-1}Q_tu=\mathcal S_*^*(G+rI_n)^{-1}Q_tu+\frac{1-p}pk_h(\cdot,X)^\top(G+rI_n)^{-1}Q_tu$.
The first term has population norm at most $\lVert u\rVert_2/(2\sqrt r)$ by~\eqref{eq:resolvent}, and the second at most $a^{q/2}\lVert u\rVert_2$ by~\eqref{eq:spike-norm}.
Their sum is $(2^{-1}r^{-1/2}+a^{q/2})\lVert u\rVert_2=C_{\mathrm{pop}}\lVert u\rVert_2/\sqrt n$.
\end{proof}

\subsection{Ridge Regression}

The ridge bounds below apply to iid observations with \rev{any input law $\rho$}.
Let $\ell$ be a measurable kernel with separable RKHS and $\ell(x,x)\le1$, and let $\gamma>0$.
The empirical ridge estimator $\widehat f_{\ell,\gamma}=\argmin_{f\in\mathcal H_\ell}\{\Rhat(f)+\gamma\lVert f\rVert_{\mathcal H_\ell}^2\}$ is $\ell(\cdot,X)^\top(L+n\gamma I_n)^{-1}y$ \cite{devito2005learning}.
This formula solves the normal equation $(\mathcal S_\ell^*\mathcal S_\ell+n\gamma I)\widehat f_{\ell,\gamma}=\mathcal S_\ell^*y$, since $(\mathcal S_\ell^*\mathcal S_\ell+sI)\mathcal S_\ell^*=\mathcal S_\ell^*(L+sI_n)$.
Let $\mathcal A_\ell(\gamma)=\inf_{f\in\mathcal H_\ell}\{\lVert f-f^\star\rVert_\rho^2+\gamma\lVert f\rVert_{\mathcal H_\ell}^2\}$.

\rev{For ridge regression, we bound expected excess risk by an approximation term and a noise term controlled by regularization.}
\rev{Holding the absolute penalty fixed, a rank-one update compares leave-one-out and full-sample residuals; empirical optimality then bounds expected excess risk.}

\begin{lem}\label{lem:ridge-oracle}
Let $x$ have any distribution, with $\lVert\cdot\rVert_\rho$ the norm of the corresponding $L^2$ space, and let $|y|\le1$, $\sigma^2=R(f^\star)$ and $c=(1+(n\gamma)^{-1})^2$.
Then $\E[\lVert\widehat f_{\ell,\gamma}-f^\star\rVert_\rho^2]\le c\mathcal A_\ell(\gamma)+(c-1)\sigma^2$.
\end{lem}

\begin{proof}
We draw $N=n+1$ observations and keep the absolute penalty fixed at $s=n\gamma$.
For $i\le N$, let $f^{(-i)}$ minimize $\sum_{j\ne i}(y_j-f(x_j))^2+s\lVert f\rVert_{\mathcal H_\ell}^2$; it has the law of $\widehat f_{\ell,\gamma}$ and is independent of $(x_i,y_i)$.
Let $f^+$ minimize the same objective summed over all $N$ observations.
\rev{For $u,v\in\mathcal H_\ell$, the rank-one operator $u\otimes v$ maps $g$ to $\langle v,g\rangle_{\mathcal H_\ell}u$.}
With $\ell_i=\ell(\cdot,x_i)$, $M_{-i}=sI+\sum_{j\ne i}\ell_j\otimes\ell_j$ and $u_i=\langle\ell_i,M_{-i}^{-1}\ell_i\rangle_{\mathcal H_\ell}\in[0,1/s]$, the Sherman--Morrison identity \cite[Lemma~9]{mourtada2022elementary}, applied on $\operatorname{span}\{\ell_1,\ldots,\ell_N\}$ containing both fits, \rev{gives
\[
f^+-f^{(-i)}
=\frac{y_i-f^{(-i)}(x_i)}{1+u_i}M_{-i}^{-1}\ell_i .
\]
Evaluating at $x_i$ and rearranging yields
$y_i-f^{(-i)}(x_i)=(1+u_i)(y_i-f^+(x_i))$.}
By optimality of $f^+$, for every deterministic $f\in\mathcal H_\ell$,
\[
\frac1N\sum_i(y_i-f^{(-i)}(x_i))^2\le\frac cN\sum_i(y_i-f^+(x_i))^2\le c\Bigl[\frac1N\sum_i(y_i-f(x_i))^2+\frac sN\lVert f\rVert_{\mathcal H_\ell}^2\Bigr].
\]
Taking expectations and using $s/N\le\gamma$ gives $\E[R(\widehat f_{\ell,\gamma})]\le c(R(f)+\gamma\lVert f\rVert_{\mathcal H_\ell}^2)$.
Subtracting $\sigma^2$ and taking the infimum over $f$ proves the claim.
\end{proof}

\rev{We now verify that the chosen sample size and ridge penalty make the global estimator accurate with the probability needed for Theorem~\ref{thm:finite}.}

\begin{lem}\label{lem:tuned-ridge}
Under the schedule~\eqref{eq:schedule} and the conditions of Theorem~\ref{thm:finite},
$\E[\lVert\widehat f_{*,\lambda}-f^\star\rVert_\rho^2]\le21\epsilon\delta/256$, and $\lVert\widehat f_{*,\lambda}-f^\star\rVert_\rho\le\sqrt\epsilon/2$ with probability at least $1-\delta/3$.
\end{lem}

\begin{proof}
Let $u=\epsilon\delta$, so that $r=64/u\ge64$, $c=(1+1/r)^2<17/16$ and $c-1\le u/32+u^2/4096<33u/1024$.
The conditions give $\mathcal A_*(\lambda)\le\lVert f_\epsilon-f^\star\rVert_\rho^2+\lambda B_*^2\le u/24$, and Lemma~\ref{lem:ridge-oracle} with $\ell=k_*$ and $\sigma^2\le1$ gives
\[
\E[\lVert\widehat f_{*,\lambda}-f^\star\rVert_\rho^2]\le\frac{17}{16}\cdot\frac u{24}+\frac{33u}{1024}<\frac{21u}{256}.
\]
Markov's inequality at level $\epsilon/4$ gives failure probability at most $21\delta/64<\delta/3$.
\end{proof}

\subsection{Ridge Concentration and Logarithmic Confidence}\label{app:ridge-concentration}

We prove the alternative ridge estimate used in Corollary~\ref{cor:log-confidence}.
Let $\ell$ be a measurable kernel with separable RKHS $\mathcal H_\ell$ and $\ell(x,x)\le1$, and let the observations be iid with input law $\rho$ and $|y|\le1$.
Write $\ell_x=\ell(\cdot,x)$ and
\[
\mathcal C=\E[\ell_x\otimes\ell_x],\qquad
\mathcal C_n=\frac1n\sum_{i=1}^n\ell_{x_i}\otimes\ell_{x_i}.
\]
These operators are positive and trace class, with $\operatorname{Tr}(\mathcal C)\le1$.
The following two lemmas specialize the matrix and Hilbert-space concentration inequalities of Tropp~\cite{tropp2015introduction} and Pinelis~\cite{pinelis1994optimum}.

\rev{The first concentration bound compares empirical and population covariances at the scale set by the ridge penalty.}
\rev{We normalize by the regularized covariance, use a dimension-free trace exponential bound, and pass to the RKHS through finite-rank projections.}

\begin{lem}\label{lem:relative-ridge-covariance}
If $0<\zeta\le1$ and $n\zeta\ge8$, then
\[
\Pr\Bigl[
\Bigl\lVert(\mathcal C+\zeta I)^{-1/2}(\mathcal C_n-\mathcal C)
(\mathcal C+\zeta I)^{-1/2}\Bigr\rVert_{\mathrm{op}}>\frac12
\Bigr]
\le\frac4\zeta e^{-n\zeta/10}.
\]
\end{lem}

\begin{proof}
First suppose the Hilbert space is finite dimensional.
Put $A=\mathcal C+\zeta I$, $b=1/\zeta$, and define
\[
z=A^{-1/2}\ell_x,\qquad U=z\otimes z,\qquad Q=\E[U],\qquad X=U-Q.
\]
We have $0\preceq U\preceq bI$, $0\preceq Q\preceq I$ and $\operatorname{Tr}(Q)\le b$, so $\lVert X\rVert_{\mathrm{op}}\le b$.
Moreover,
\[
V:=\E[X^2]=\E[U^2]-Q^2\preceq bQ,\qquad
\lVert V\rVert_{\mathrm{op}}\le b,\qquad \operatorname{Tr}(V)\le b^2 .
\]
For independent copies $X_i$, put $S_n=\sum_iX_i$ and, for $0<\theta<3/b$, set
\[
g_\theta=\frac{\theta^2}{2(1-\theta b/3)},\qquad
F_\theta(S)=\operatorname{Tr}(e^{\theta S}-\theta S-I).
\]
The matrix \rev{trace exponential} bound in \cite[Lemma~6.6.2 and the proof of Theorem~7.7.1]{tropp2015introduction} gives
\[
\E[F_\theta(S_n)]
\le\operatorname{Tr}(e^{ng_\theta V}-I)
\le\frac{\operatorname{Tr}(V)}{b}(e^{ng_\theta b}-1)
\le b e^{ng_\theta b}.
\]
\rev{Subtracting $I$ cancels the ambient-dimension term in the trace exponential bound, and centering gives $\E[\operatorname{Tr}(S_n)]=0$.}
The second inequality uses convexity on $[0,b]$.
The same bound holds for $-S_n$.
Since $e^u-u-1\ge e^u/2$ for $u\ge2$, Markov's inequality gives
\[
\Pr[\lVert S_n\rVert_{\mathrm{op}}>s]
\le4b\exp(ng_\theta b-\theta s),\qquad \theta s\ge2.
\]
Choose $\theta=1/(2b)$ and $s=n/2$.
Then $\theta s=n\zeta/4\ge2$ and $ng_\theta b-\theta s=-n\zeta/10$.
Since $S_n/n=A^{-1/2}(\mathcal C_n-\mathcal C)A^{-1/2}$, this proves the finite-dimensional claim.

For a separable RKHS, let $P_j$ be finite-rank orthogonal projections increasing strongly to the identity, and apply the argument to $P_jX_iP_j$.
Their norms are at most $b$, their means are zero, and
\[
P_jVP_j-\E[(P_jXP_j)^2]
=\E[P_jX(I-P_j)XP_j]\succeq0.
\]
Their second moments therefore satisfy the same norm and trace bounds, uniformly in $j$.
Each $X_i$ is trace class, hence $S_n$ is compact and $P_jS_nP_j\to S_n$ in operator norm.
The strict event $\lVert S_n\rVert_{\mathrm{op}}>n/2$ implies the corresponding compressed event for all sufficiently large $j$.
Fatou's lemma for the event indicators proves the claim.
\end{proof}

\rev{The second bound controls sums of independent, mean-zero random vectors using their variance and maximum size.}
\rev{An exponential martingale bound reduces the tail estimate to the variance and increment bound; choosing its parameter gives the stated threshold.}

\begin{lem}\label{lem:hilbert-ridge-score}
Let $\xi_1,\ldots,\xi_n$ be independent, mean-zero random elements of a separable real Hilbert space, with $\lVert\xi_i\rVert\le b$ almost surely, and put $V=\sum_i\E[\lVert\xi_i\rVert^2]$.
For every $s>0$,
\[
\Pr\Bigl[\Bigl\lVert\sum_i\xi_i\Bigr\rVert>\sqrt{2\mathrm eVs}+2bs\Bigr]\le2e^{-s}.
\]
\end{lem}

\begin{proof}
The case $V=0$ is immediate.
Otherwise $V>0$ and $b>0$.
We apply \cite[Theorem~3.1]{pinelis1994optimum} to the partial-sum martingale $F_j=\sum_{i=1}^{\min\{j,n\}}\xi_i$ with its natural filtration.
Its increments are independent and bounded\rev{.
The parallelogram identity $\lVert u+v\rVert^2+\lVert u-v\rVert^2=2\lVert u\rVert^2+2\lVert v\rVert^2$ verifies the theorem's smoothness condition with constant one.}
The resulting bound on $\sup_j\{\lVert F_j\rVert\}$ gives, for $0<\theta\le1/b$,
\begin{align*}
\Pr\Bigl[\Bigl\lVert\sum_i\xi_i\Bigr\rVert>t\Bigr]
&\le2\exp\Bigl(-\theta t+
\sum_i\E[e^{\theta\lVert\xi_i\rVert}-1-\theta\lVert\xi_i\rVert]\Bigr)\\
&\le2e^{-\theta t+\mathrm e\theta^2V/2}.
\end{align*}
The last step uses $e^u-1-u\le\mathrm e u^2/2$ for $0\le u\le1$.
Choosing $\theta=\min\{t/(\mathrm eV),1/b\}$ bounds this by
$2\exp(-\min\{t^2/(2\mathrm eV),t/(2b)\})$.
At $t=\sqrt{2\mathrm eVs}+2bs$, both terms in the minimum are at least $s$.
\end{proof}

Let $f_{\ell,\zeta}$ minimize $R(f)+\zeta\lVert f\rVert_{\mathcal H_\ell}^2$, where $0<\zeta\le1$.
Comparison with zero gives $R(f_{\ell,\zeta})+\zeta\lVert f_{\ell,\zeta}\rVert_{\mathcal H_\ell}^2\le1$.
Define
\[
Z=(y-f_{\ell,\zeta}(x))(\mathcal C+\zeta I)^{-1/2}\ell_x.
\]
Then $\E[\lVert Z\rVert^2]\le1/\zeta$ and $\lVert Z\rVert\le(1+\zeta^{-1/2})/\sqrt\zeta\le2/\zeta$.
\rev{Let $Z_i$ be $Z$ evaluated at observation $(x_i,y_i)$.}
Thus $Z_i-\E[Z]$ has norm at most $4/\zeta$ and second moment at most $1/\zeta$.
Lemma~\ref{lem:hilbert-ridge-score} yields, with failure probability at most $2e^{-s}$,
\begin{equation}\label{eq:centered-ridge-score}
\Bigl\lVert\frac1n\sum_iZ_i-\E[Z]\Bigr\rVert
\le\sqrt{\frac{2\mathrm e s}{n\zeta}}+\frac{8s}{n\zeta}.
\end{equation}

Put $A_\zeta=\mathcal C+\zeta I$ and $M_\zeta=A_\zeta^{-1/2}(\mathcal C_n-\mathcal C)A_\zeta^{-1/2}$.
The population normal equation gives $\E[(y-f_{\ell,\zeta}(x))\ell_x]=\zeta f_{\ell,\zeta}$.
We subtract it from the empirical equation and conjugate by $A_\zeta^{-1/2}$ to obtain
\[
A_\zeta^{1/2}(\widehat f_{\ell,\zeta}-f_{\ell,\zeta})
=(I+M_\zeta)^{-1}\Bigl(\frac1n\sum_iZ_i-\E[Z]\Bigr).
\]
On $\lVert M_\zeta\rVert_{\mathrm{op}}\le1/2$, the inverse has norm at most two.
Since $\mathcal C\preceq A_\zeta$, we have $\lVert f\rVert_\rho\le\lVert A_\zeta^{1/2}f\rVert_{\mathcal H_\ell}$.
Population optimality and~\eqref{eq:centered-ridge-score} therefore give
\begin{equation}\label{eq:concentrated-ridge-risk}
\lVert\widehat f_{\ell,\zeta}-f^\star\rVert_\rho
\le\sqrt{\mathcal A_\ell(\zeta)}
+2\sqrt{\frac{2\mathrm e s}{n\zeta}}+\frac{16s}{n\zeta}.
\end{equation}
By a union bound, this estimate holds with failure probability at most $4\zeta^{-1}e^{-n\zeta/10}+2e^{-s}$ for $0<\zeta\le1$ and $n\zeta\ge8$.

\rev{These estimates give the required ridge accuracy with a sample-size bound depending only logarithmically on the inverse failure probability.}

\begin{lem}\label{lem:log-ridge}
Under the approximant, sample-size and schedule conditions of Corollary~\ref{cor:log-confidence},
\[
\Pr[\lVert\widehat f_{*,\lambda_{\log}}-f^\star\rVert_\rho
\le\sqrt\epsilon/2]\ge1-\delta/3.
\]
\end{lem}

\begin{proof}
We first verify that~\eqref{eq:log-sample-size} implies $n\ge C_{\log}\log(24n/\delta)$.
Set $C=C_{\log}$, $a_0=24/\delta$, $s_0=\log(2a_0C)>1$ and $x_0=2Cs_0$.
Then $\log(a_0x_0)=s_0+\log(s_0)\le2s_0=x_0/C$, and $x-C\log(a_0x)$ is increasing for $x\ge C$.
This proves the claim for every integer $n\ge\lceil x_0\rceil$.
Writing $L=L_{\log}$, $r=r_{\log}$ and $\lambda=\lambda_{\log}$, we obtain
\[
n\ge2r,\qquad \lambda\le1/2,\qquad n\ge40rB_{\log}^2/\epsilon.
\]
Consequently $\mathcal A_*(\lambda)\le\epsilon/32+rB_{\log}^2/n\le\epsilon/16$.
Since $r\ge8$ and $r/10\ge L$, Lemma~\ref{lem:relative-ridge-covariance} gives covariance failure probability at most
\[
\frac{4n}{r}e^{-r/10}\le\frac{4n}{r}e^{-L}
=\frac{\delta}{6r}\le\delta/6.
\]
Taking $s=L$ and $\zeta=\lambda$ in~\eqref{eq:centered-ridge-score} gives score failure probability at most $2e^{-L}=\delta/(12n)\le\delta/6$.
On their intersection, \eqref{eq:concentrated-ridge-risk} yields
\[
\lVert\widehat f_{*,\lambda}-f^\star\rVert_\rho
\le\frac{\sqrt\epsilon}4+\frac{\sqrt{\mathrm e}}8\sqrt\epsilon
+\frac{\epsilon}{32}<\frac{\sqrt\epsilon}2,
\]
using $\sqrt{\mathrm e}<7/4$ and $\epsilon<1$.
\end{proof}

\section{Proofs for Finite-Width Networks}\label{app:network}

\rev{We establish residual decay, convergence to interpolation and uniform comparison with kernel flow, then prove the results of Section~\ref{subsec:main}.}

Throughout this appendix $d\ge2$, $n\ge16$, $4\le r\le n$, $0<a<1/2$, $h=a^2/8<1/32$, $q=d-1$, and $v,p,\nu$ are given by~\eqref{eq:mixture}; further $\lambda=r/n$ and $H=n/\nu$.
The width $m$ is even, and the paired initialization of Section~\ref{sec:localized} gives $f_{W(0)}=0$.
We write $z_i=z_{x_i}$, so $\lVert z_x\rVert_2=\sqrt2$, and $\mathbf r(t)=y-f_{W(t)}(X)$ for the training residual of the network.

\subsection{Tangent Kernels and Initialization Bounds}\label{app:ntk}

An \emph{allowed gate array} for $W$ is a matrix $\vartheta\in[0,1]^{m\times n}$ with $\vartheta_{ji}=1$ if $\theta_j^\top z_i>0$ and $\vartheta_{ji}=0$ if $\theta_j^\top z_i<0$.
The corresponding prediction Jacobian $J_\vartheta$ has the $(i,j)$ block $s_j\vartheta_{ji}z_i^\top/\sqrt m$, the training tangent matrix is $K_\vartheta=J_\vartheta J_\vartheta^\top$, with entries $\frac{z_i^\top z_\ell}m\sum_j\vartheta_{ji}\vartheta_{j\ell}$, and the population-to-training cross-kernel is
\[
\widetilde k_\vartheta(x,x_i)=\frac{z_x^\top z_i}m\sum_{j=1}^m\mathbf{1}\{\theta_j^\top z_x>0\}\vartheta_{ji}.
\]
Let $\chi_{j,0}(x)=\mathbf{1}\{\theta_j(0)^\top z_x>0\}$ and $k_{m,0}(x,u)=\frac{z_x^\top z_u}m\sum_j\chi_{j,0}(x)\chi_{j,0}(u)$ be the empirical kernel at initialization.
Since the rows come in identical pairs, $k_{m,0}$ is an average over the $m/2$ independent representatives, and its mean is $k=pk_*+(1-p)k_h$: a global representative contributes $z_x^\top z_u\mathbf{1}\{Z^\top x>1\}\mathbf{1}\{Z^\top u>1\}$ and a local one $z_x^\top z_u\mathbf{1}\{C^\top x>1-h\}\mathbf{1}\{C^\top u>1-h\}$.
Both kernels are continuous, $k_*\le a_*<1$ and $k_h\le2v\le1$, so $|k|\le pa_*+\nu\le1$.
Kernel flow is $\mathbf r_k(t)=e^{-2tK/n}y$ and $f^k_t=\frac2n\int_0^tk(\cdot,X)^\top\mathbf r_k(s)\,ds$.

With $c_d,\omega_d,c_{\mathrm{cap}}$ from Lemma~\ref{lem:caps}, put
\[
\bar c_d=c_d3^{\max\{\omega_d,0\}}2^{\max\{-\omega_d,0\}},
\]
and define
\begin{gather}
D_G=4p\sqrt{2/\pi},\qquad D_L=8(1-p)\bar c_dh^{\omega_d},\qquad \zeta=\bigl(\sqrt{3/2}a\bigr)^q,\notag\\
C_{\mathrm{pop}}=\frac1{2\sqrt\lambda}+\sqrt{na^q},\qquad \mathcal E=64D_G^2H^2+\frac{32D_L^2H^2}{n^3}+144p+\frac{80}n,\notag\\
\mathcal B=64D_G^2H^2+\frac{64\zeta D_L^2H^2}n+144p+144\zeta .\label{eq:constants}
\end{gather}
\rev{For $\eta\in(0,1]$ and $\delta_0\in(0,1)$, we define $m_0(n,d,r,a,\eta,\delta_0)$ as the smallest even integer $m$ satisfying}
\begin{equation}\label{eq:width}
m\ge\max\Bigl\{\frac{64H^2}{nh^2},\ \frac{24H^2\bigl[(1+C_{\mathrm{pop}}^2)\mathcal E+\mathcal B\bigr]}{\delta_0\eta^2}\Bigr\}.
\end{equation}

\subsection{Gradient Flow}\label{app:flow}

Let $\mathcal L(W)=\Rhat(f_W)$.
A gradient-flow solution is a locally absolutely continuous path with $\dot W(t)\in-\partial^C\mathcal L(W(t))$ for almost every $t\ge0$.

\rev{The following lemma guarantees gradient flow for all training times and describes how its predictions and training loss evolve.}
\rev{Local existence and the energy identity give global solutions; the chain rule gives the evolution equations, and Euler limits admit a measurable selection.}

\begin{lem}\label{lem:clarke-flow}
For every initial matrix there is a global solution.
Every solution has measurable allowed gate arrays $\vartheta(t)$ such that, for almost every $t$, with $K_t=K_{\vartheta(t)}$,
\begin{equation}\label{eq:clarke}
\dot\theta_j=\frac{2s_j}{n\sqrt m}\sum_{i=1}^n\mathbf r_i\vartheta_{ji}z_i,\qquad \dot{\mathbf r}=-\frac2nK_t\mathbf r,\qquad \frac d{dt}\mathcal L(W(t))=-\lVert\dot W(t)\rVert_{\mathrm{F}}^2 .
\end{equation}
For fixed $n,m,d$, a global solution can be chosen Borel measurably in the data and the initialization, with the path space equipped with locally uniform convergence.
\end{lem}

\begin{proof}
The loss is locally Lipschitz and piecewise quadratic: the hyperplanes $\{\theta_j^\top z_i=0\}$ cut the parameter space into open polyhedral cells, on each of which $\mathcal L$ is a polynomial with gradient $-\frac2nJ_\vartheta^\top\mathbf r$ for the binary gate array of the cell.
By the Clarke-gradient formula \cite[Eq.~(37)]{Cortes2008}, $\partial^C\mathcal L(W)$ is the convex hull of the gradients of the cells incident to $W$, and since $J_\vartheta$ is linear in $\vartheta$, every element has the form $-\frac2nJ_\vartheta^\top\mathbf r$ for an allowed gate array.\footnote{The locators for Cort\'es refer to the expanded version, \arxiv{0901.3583v1}.}
Along any solution, the chain rule for semialgebraic functions \cite[Theorem~5.8]{DavisEtAl2018} gives $\frac d{dt}\rev{\phi(\theta_j^\top z_i)}=\vartheta_{ji}\dot\theta_j^\top z_i$ almost everywhere for every allowed value of a zero gate.
The finitely many training preactivations share a common exceptional null set, so $\frac d{dt}f_{W(t)}(X)=J_\vartheta\dot W$; substituting $\dot W=\frac2nJ_\vartheta^\top\mathbf r$ gives the row and residual equations in~\eqref{eq:clarke}.
Differentiating $\mathcal L=\lVert\mathbf r\rVert_n^2$ now gives $\mathcal L'=-4\mathbf r^\top K_t\mathbf r/n^2=-\lVert\dot W\rVert_{\mathrm F}^2$, proving the energy identity.
Results of \cite{DavisEtAl2018} are cited with their numbering in the arXiv version \arxiv{1804.07795v3}.
Let $\vartheta(t,u)$ use the forced gates at nonzero preactivations and the coordinates of $u\in[0,1]^{mn}$ at zeros.
We choose a measurable gate array by minimizing
$u\mapsto\lVert\frac2nJ_{\vartheta(t,u)}^\top\mathbf r(t)-\dot W(t)\rVert_{\mathrm{F}}^2$
over this cube, using a Borel representative of $\dot W$.
This objective is Borel in $t$ and continuous in $u$; its nonempty compact minimizer sets have a Borel graph, and Arsenin--Kunugui \cite[p.~297]{Kechris1995} supplies a Borel choice.
The minimum is zero almost everywhere.

For fixed data, the \rev{subdifferential} is nonempty, compact, convex, locally bounded and upper semicontinuous \cite[Proposition~6]{Cortes2008}; the autonomous inclusion therefore has local solutions \cite[Sidebar~8, Proposition~S1]{Cortes2008}.
The energy identity gives $\lVert W(t)-W(s)\rVert_{\mathrm{F}}\le\sqrt{(t-s)\mathcal L(W(s))}$.
Thus a solution is Cauchy at a finite maximal time, and concatenating with a local solution from its limiting matrix extends it globally.

For a Borel path selection, let \rev{$b(S_n,W)$} be the negative gradient of the first incident cell in a fixed order of binary gate arrays.
Incidence is decided by sign tests and a strict linear system for a perturbation of the zero preactivations; every feasible strict system has a rational witness, so $b$ is jointly Borel.
\rev{With $W_0=W(0)$, the} Euler polygons $P_k$ of mesh $1/k$ are Borel in \rev{$(S_n,W_0)$}.
The bound $\lVert b\rVert_{\mathrm{F}}\le2\sqrt2\lVert\mathbf r\rVert_n\le2\sqrt2+4\lVert W\rVert_{\mathrm{F}}$ and discrete Gronwall's inequality bound their values and slopes on each finite time interval, so they are locally uniformly bounded and equicontinuous.
\rev{Fix the data and a finite horizon, and take a subsequence $P_k\to P$ uniformly, with left-endpoint step interpolants $\bar P_k$.
The uniformly bounded velocities have a subsequence converging weakly in $L^2$ to $\dot P$, and $\bar P_k\to P$ uniformly.
For $F(W)=-\partial^C\mathcal L(W)$ and each fixed matrix $V$, define $s_V(W)=\sup_{U\in F(W)}\{\langle V,U\rangle_{\mathrm F}\}$, where $\langle\cdot,\cdot\rangle_{\mathrm F}$ is the Frobenius inner product.
The polygon construction gives
\[
\langle V,\dot P_k(t)\rangle_{\mathrm F}\le s_V(\bar P_k(t))
\quad\text{for almost every }t.
\]
We integrate against bounded nonnegative test functions and use weak convergence and the upper semicontinuity and local boundedness of $s_V$ to obtain
$\langle V,\dot P(t)\rangle_{\mathrm F}\le s_V(P(t))$ almost everywhere.
We first take rational matrices $V$ to obtain a common null set, then all $V$ by continuity.
Since $F(P(t))$ is closed and convex, these inequalities imply $\dot P(t)\in F(P(t))$ almost everywhere, so every cluster limit solves the inclusion.}
Arzel\`a--Ascoli therefore gives a nonempty compact set of global solution cluster points.
In a metric $d_{\mathcal P}$ for locally uniform convergence on continuous paths, its graph is the Borel set
\[
\bigcap_{j,N\ge1}\ \bigcup_{k\ge N}
\{\rev{(S_n,W_0,P)}\mid d_{\mathcal P}(\rev{P_k(S_n,W_0)},P)<1/j\}.
\]
The data and initial matrices form a standard Borel space, and the path space is Polish, so Arsenin--Kunugui \cite[p.~297]{Kechris1995} gives the asserted Borel selection.
\end{proof}

Joint evaluation of a selected path and time is Borel, so measurable choices of training time define measurable predictors.
We define the limiting learner by $\lim_{k\to\infty}W(k)$ on the Borel set where the integer-time sequence converges, and by zero otherwise.
On the events where continuous-time convergence is proved below, this agrees with $W(\infty)$.

For each fixed $x\in\Sph$, the scalar chain rule gives, for almost every $t$,
\[
\frac d{dt}f_{W(t)}(x)=\frac1{\sqrt m}\sum_js_j\mathbf{1}\{\theta_j(t)^\top z_x>0\}\dot\theta_j(t)^\top z_x .
\]
The exceptional null set may depend on $x$, but the displayed integrand is jointly measurable and bounded by $\sqrt2\lVert\dot W(t)\rVert_{\mathrm{F}}$, which is integrable on compact time intervals.
The scalar fundamental theorem, Fubini and~\eqref{eq:clarke} therefore give, in $L^2(\rho)$,
\begin{equation}\label{eq:population-chain}
f_{W(t)}-f_{W(0)}=\frac2n\int_0^t\sum_{i=1}^n\widetilde k_{\vartheta(u)}(\cdot,x_i)\mathbf r_i(u)\,du .
\end{equation}

\subsection{Initialization Certificates}\label{app:certificates}

Let $\mathsf t_j\in\{G,L\}$ be the type of row $j$ at initialization, and define the movement radii and slab half-widths
\[
M_G=\frac{2\sqrt2H}{\sqrt m},\qquad M_L=\frac{2\sqrt2H}{\sqrt{nm}},\qquad \beta_J=\sqrt2M_J\quad(J\in\{G,L\}).
\]
Let $\mathcal W$ be the set of parameter matrices whose rows satisfy $\lVert\theta_j-\theta_j(0)\rVert_2\le M_{\mathsf t_j}$.
If $W\in\mathcal W$, a gate at $x$ can differ from its initial value only if $|\theta_j(0)^\top z_x|\le\beta_{\mathsf t_j}$.
For the geometric bounds below, we impose the first width bound in~\eqref{eq:width}, $m\ge64H^2/(nh^2)$, which gives $\beta_L=4H/\sqrt{nm}\le h/2$.

\emph{One Training Point per Local Row.}
Let $W\in\mathcal W$ and let row $j$ be local, with initial direction $C_j$.
If row $j$ has a nonzero allowed gate at $x$, then $C_j^\top x\ge(1-h)(1-\beta_L)\ge1-3h/2$, so $\lVert C_j-x\rVert\le\sqrt{3h}$.
On $E_{\mathrm{sep}}$, two training inputs within $\sqrt{3h}$ of $C_j$ would be at distance at most $2\sqrt{3h}=\sqrt{3/2}a<2a$, which is impossible.
Thus the ball around the initial direction $C_j$ contains at most one possible training input, denoted $i(j)$ if it exists; this index is fixed throughout $\mathcal W$.
The row's contribution to $\widetilde k_\vartheta(x,x_i)$ vanishes unless $\lVert x-x_i\rVert\le b_{\mathrm{loc}}:=\sqrt{3/2}a<1$; the ball of this radius has mass at most $\zeta$ by Lemma~\ref{lem:caps}.

For $J\in\{G,L\}$ define the slab fractions and their norms
\begin{gather*}
N_J(x)=\frac1m\sum_{j=1}^m\mathbf{1}\{\mathsf t_j=J,\ |\theta_j(0)^\top z_x|\le\beta_J\},\\
N_{J,\mathrm{tr}}^2=\frac1n\sum_iN_J(x_i)^2,\qquad N_{J,\rho}=\lVert N_J\rVert_\rho .
\end{gather*}
Split $k_{m,0}=k_{G,m,0}+k_{L,m,0}$ by row type, with Gram matrices $K_{G,m,0},K_{L,m,0}$, and define the initialization errors
\begin{gather*}
\Gamma_G=\frac{\lVert K_{G,m,0}-pG\rVert_{\mathrm{op}}}n,\qquad \Gamma_L=\frac{\lVert K_{L,m,0}-\nu I_n\rVert_{\mathrm{op}}}n,\\
\Pi_G^2=\frac1n\sum_i\lVert k_{G,m,0}(\cdot,x_i)-pk_*(\cdot,x_i)\rVert_\rho^2,\\
\Pi_L^2=\frac1n\sum_i\lVert k_{L,m,0}(\cdot,x_i)-(1-p)k_h(\cdot,x_i)\rVert_\rho^2 .
\end{gather*}
Let $\Pi_{\mathrm{drift}}^2=\frac1n\sum_i\int4[N_L(x)+N_L(x_i)]^2\mathbf{1}\{\lVert x-x_i\rVert\le b_{\mathrm{loc}}\}\,d\rho(x)$, and define the certificates
\begin{equation}\label{eq:certificates}
\begin{aligned}
\mathfrak e&=4N_{G,\mathrm{tr}}+\frac2n\max_i\{N_L(x_i)\}+\Gamma_G+\Gamma_L,\\
\mathfrak b&=2N_{G,\rho}+2N_{G,\mathrm{tr}}+\Pi_{\mathrm{drift}}+\Pi_G+\Pi_L .
\end{aligned}
\end{equation}
These certificates are measurable in the inputs and initialization, since their definitions use finite sums, matrix norms and integrals of bounded jointly measurable functions.

\rev{The next lemma bounds the expected fraction of neurons whose gates may change within the prescribed movement radii.}

\begin{lem}\label{lem:slab}
Let $m\ge64H^2/(nh^2)$.
The means $\pi_J=\E[N_J(x)]$ do not depend on $x$, and
\[
\pi_G\le\bar\pi_G:=\min\{p,D_GH/\sqrt m\},\qquad \pi_L\le\bar\pi_L:=\min\{1,D_LH/\sqrt{nm}\}.
\]
\end{lem}

\begin{proof}
Rotation invariance gives the first claim.
For a global row, $\theta_j(0)^\top z_x=Z^\top x-1\sim N(-1,1)$, so the slab has probability at most $\sqrt{2/\pi}\beta_G$; multiplying by $p$ gives $\pi_G\le D_GH/\sqrt m$.
For a local row, the slab is $|C^\top x-(1-h)|\le(1-h)\beta_L$, an interval of length at most $2\beta_L$ inside $[1-3h/2,1-h/2]$.
On this interval $h/2\le1-t^2\le3h$, so the density satisfies $c_d(1-t^2)^{\omega_d}\le\bar c_dh^{\omega_d}$ by the definition of $\bar c_d$.
Multiplying by $1-p$ gives $\pi_L\le2(1-p)\bar c_dh^{\omega_d}\beta_L=D_LH/\sqrt{nm}$.
\end{proof}

\rev{The certificates introduced above control the network's kernel errors near initialization and are small on average at sufficient width.}
\rev{Gate changes are confined to thin slabs and local rows affect at most one training point; independent representatives give the moment bounds.}

\begin{lem}\label{lem:certificates}
Fix inputs in $E_{\mathrm{sep}}$ and let $m\ge64H^2/(nh^2)$.
For every $W\in\mathcal W$ and every allowed gate array $\vartheta$,
\begin{equation}\label{eq:drift}
\frac{\lVert K_\vartheta-K\rVert_{\mathrm{op}}}n\le\mathfrak e,\qquad \Bigl(\frac1n\sum_i\lVert\widetilde k_\vartheta(\cdot,x_i)-k(\cdot,x_i)\rVert_\rho^2\Bigr)^{1/2}\le\mathfrak b .
\end{equation}
Moreover, over the initialization, $\E[\mathfrak e^2]\le\mathcal E/m$ and $\E[\mathfrak b^2]\le\mathcal B/m$.
\end{lem}

\begin{proof}
\emph{Drift.}
For $u,w,u_0,w_0\in[0,1]$ we have $|uw-u_0w_0|\le|u-u_0|+|w-w_0|$; moreover, $|z_x^\top z_u|\le2$.
Hence the global rows change the $(i,\ell)$ entry of the tangent matrix by at most $2(N_G(x_i)+N_G(x_\ell))$ relative to $K_{G,m,0}$, and the Frobenius norm of this change, divided by $n$, is at most $4N_{G,\mathrm{tr}}$.
The local rows keep the tangent matrix diagonal, by the one-point property, and change its $i$-th diagonal entry by at most $2N_L(x_i)$.
Together with $\Gamma_G,\Gamma_L$, this proves the first bound in~\eqref{eq:drift}.
In the same way, the global rows change $\widetilde k_\vartheta(x,x_i)$ by at most $2N_G(x)+2N_G(x_i)$ relative to $k_{G,m,0}(x,x_i)$, and the local rows change it by at most $2[N_L(x)+N_L(x_i)]$; the local change vanishes when $\lVert x-x_i\rVert>b_{\mathrm{loc}}$.
\rev{Minkowski's inequality and the definitions of $\Pi_G$ and $\Pi_L$ prove the second bound.}

\emph{Moments.}
At each fixed $x$, $N_J(x)$ is an average of $m/2$ independent Bernoulli variables with mean $\pi_J$ (each representative counted twice), so $\E[N_J(x)^2]\le\bar\pi_J^2+2\bar\pi_J/m=:\bar a_J$, and by Tonelli $\E[N_{J,\mathrm{tr}}^2],\E[N_{J,\rho}^2]\le\bar a_J$.
Centering before taking the maximum gives $\E[\max_i\{N_L(x_i)^2\}]\le2\bar\pi_L^2+2\sum_i\E[(N_L(x_i)-\pi_L)^2]\le2\bar\pi_L^2+4n\bar\pi_L/m$.
Since each ball of radius $b_{\mathrm{loc}}$ has mass at most $\zeta$,
\begin{align*}
\E[\Pi_{\mathrm{drift}}^2]
&\le\frac8n\sum_i\int_{\lVert x-x_i\rVert\le b_{\mathrm{loc}}}(\E[N_L(x)^2]+\E[N_L(x_i)^2])\,d\rho\\
&\le16\zeta\bar a_L,
\end{align*}
and $\E[(2N_{G,\rho}+2N_{G,\mathrm{tr}})^2]\le16\bar a_G$.
For one independent representative with initial gate $\chi_0(x)=\mathbf{1}\{\theta(0)^\top z_x>0\}$, define the type-marked kernel feature
\[
F_J(x,u)=\mathbf{1}\{\mathsf t=J\}(1+x^\top u)\chi_0(x)\chi_0(u),
\qquad J\in\{G,L\}.
\]
Then $0\le F_J\le2$, $\E[F_G]=pk_*\le p$ and $\E[F_L]=(1-p)k_h\le\nu$.
Since each empirical type kernel averages $m/2$ such features, $F_G^2\le2F_G$ gives entrywise variance at most $4p/m$, and $\E[\Gamma_G^2],\E[\Pi_G^2]\le4p/m$.
On $E_{\mathrm{sep}}$ the matrix $K_{L,m,0}$ is diagonal, with entries twice an average of $m/2$ indicators of mean $\nu/2$, so $\E[\Gamma_L^2]\le n\cdot4\nu/(n^2m)=4\nu/(nm)$; the local population kernel has mean at most $\nu$ and column support of mass at most $a^q$, so $\E[\Pi_L^2]\le4\nu a^q/m$.
Using $(\sum_{l=1}^4c_l)^2\le4\sum_lc_l^2$ for both certificates, with the two global terms of $\mathfrak b$ grouped,
\begin{align*}
\E[\mathfrak e^2]&\le64\bar\pi_G^2+\frac{32\bar\pi_L^2}{n^2}+\frac{128\bar\pi_G+64\bar\pi_L/n+16p+16\nu/n}m,\\
\E[\mathfrak b^2]&\le64\bar\pi_G^2+64\zeta\bar\pi_L^2+\frac{128\bar\pi_G+128\zeta\bar\pi_L+16p+16\nu a^q}m .
\end{align*}
Lemma~\ref{lem:slab}, $\bar\pi_G\le p$, $\bar\pi_L\le1$, $\nu\le1$ and $a^q\le\zeta$ give the claimed bounds.
\end{proof}

Define the initialization certificate $\mathcal Z_m=(2H\mathfrak e)^2+(2H\mathfrak b+2HC_{\mathrm{pop}}\mathfrak e)^2$ and the event $E_{\mathrm{init}}(\eta)=\{\mathcal Z_m\le\eta^2\}$, which depends only on the inputs and the initialization.

\rev{Sufficient width also makes the initialization certificates small with high probability on separated samples.}

\begin{lem}\label{lem:init-event}
Let $m\ge m_0(n,d,r,a,\eta,\delta_0)$.
Conditionally on any inputs in $E_{\mathrm{sep}}$, the event $E_{\mathrm{init}}(\eta)$ fails with probability at most $\delta_0/3$ over the initialization.
On $E_{\mathrm{init}}(\eta)$, $e_0:=H\mathfrak e\le\eta/2\le1/2$.
\end{lem}

\begin{proof}
By Lemma~\ref{lem:certificates},
\begin{align*}
\E[\mathcal Z_m]
&\le4H^2(1+2C_{\mathrm{pop}}^2)\E[\mathfrak e^2]+8H^2\E[\mathfrak b^2]\\
&\le8H^2[(1+C_{\mathrm{pop}}^2)\mathcal E+\mathcal B]/m
\le\delta_0\eta^2/3,
\end{align*}
and Markov's inequality gives the first claim.
The second follows from $(2H\mathfrak e)^2\le\mathcal Z_m$.
\end{proof}

\subsection{Residual Decay and Convergence}

\rev{On the separation and initialization events, we prove exponential decay of the training residual and convergence to an interpolant without leaving the initial neighborhood.}
\rev{We adapt the first-exit argument \cite[Lemmas~3.3--3.4]{du2019iclr} to type-specific radii: integrating residual decay from the kernel gap rules out exit and gives convergence.}

\begin{lem}\label{lem:residual}
Let $m\ge m_0(n,d,r,a,\eta,\delta_0)$.
On $E_{\mathrm{sep}}\cap E_{\mathrm{init}}(\eta)$, for every label vector in $[-1,1]^n$ and every solution, $W(t)\in\mathcal W$ for all $t\ge0$ and
\[
\lVert\mathbf r(t)\rVert_n\le e^{-t/H},\qquad \lVert\theta_j(t)-\theta_j(0)\rVert_2\le M_{\mathsf t_j}(1-e^{-t/H}),\qquad \int_0^\infty\lVert\mathbf r(t)\rVert_n\,dt\le H .
\]
Moreover, $W(t)$ converges to a limit $W(\infty)$, $f_{W(t)}\to f_{W(\infty)}$ uniformly on $\Sph$, and $f_{W(\infty)}(x_i)=y_i$ for all $i$.
\end{lem}

\begin{proof}
Suppose that some row of type $J$ first reaches distance $M_J$ from its initialization at a finite time $t_0$.
For $t\le t_0$, Lemma~\ref{lem:certificates} gives $K_t\succeq K-n\mathfrak eI_n\succeq\nu(1-e_0)I_n\succeq(\nu/2)I_n$, so by~\eqref{eq:clarke}
\[
\frac d{dt}\lVert\mathbf r(t)\rVert_n^2=-\frac4{n^2}\mathbf r^\top K_t\mathbf r\le-\frac2H\lVert\mathbf r(t)\rVert_n^2,
\]
and $\lVert\mathbf r(t)\rVert_n\le e^{-t/H}$, since $f_{W(0)}=0$ and $\lVert y\rVert_n\le1$.
By~\eqref{eq:clarke}, $\lVert\dot\theta_j\rVert_2\le\frac{2\sqrt2}{\sqrt m}\lVert\mathbf r\rVert_n$, and for a local row only the fixed coordinate $i(j)$ can enter, so $\lVert\dot\theta_j\rVert_2\le\frac{2\sqrt2}{n\sqrt m}|\mathbf r_{i(j)}|\le\frac{2\sqrt2}{\sqrt{nm}}\lVert\mathbf r\rVert_n$.
If no such index exists, the row velocity is zero.
Integrating up to $t\le t_0$ bounds every displacement by $M_{\mathsf t_j}(1-e^{-t/H})<M_{\mathsf t_j}$, a contradiction.
So the bounds hold for all times, and integrating the residual bound gives the last one.
The row velocities are integrable on $[0,\infty)$, so $W(t)$ converges.
Since $\sup_x\{|f_W(x)-f_V(x)|\}\le\sqrt2\lVert W-V\rVert_{\mathrm{F}}$, the networks converge uniformly, and the vanishing residual gives interpolation.
The same initialization certificates apply to every label vector and every solution.
\end{proof}

\subsection{Uniform Comparison with Kernel Flow}

\rev{The following comparison bounds the network's deviation from kernel flow throughout training, both on the sample and in population norm.}
\rev{Variation of constants expresses residual error through kernel drift; integrating the semigroup and using Lemma~\ref{lem:pop-operator} prevents error accumulation over time.}

\begin{lem}\label{lem:coupling}
Under the conditions of Lemma~\ref{lem:residual}, on $E_{\mathrm{sep}}\cap E_{\mathrm{init}}(\eta)$, for every label vector in $[-1,1]^n$ and every solution,
\[
\sup_{t\ge0}\{\lVert\mathbf r(t)-\mathbf r_k(t)\rVert_n\}\le2H\mathfrak e\le\eta,\qquad \sup_{t\ge0}\{\lVert f_{W(t)}-f^k_t\rVert_\rho\}\le2H\mathfrak b+2HC_{\mathrm{pop}}\mathfrak e\le\eta .
\]
\end{lem}

\begin{proof}
Write $\Delta_u=K_u-K$, so that $\lVert\Delta_u\rVert_{\mathrm{op}}\le n\mathfrak e$ for almost every $u$ by Lemmas~\ref{lem:certificates} and~\ref{lem:residual}.
Variation of constants gives
\begin{equation}\label{eq:duhamel}
\mathbf r(t)-\mathbf r_k(t)=-\frac2n\int_0^te^{-2(t-u)K/n}\Delta_u\mathbf r(u)\,du,
\end{equation}
and since the semigroup is contractive, $\lVert\mathbf r(t)-\mathbf r_k(t)\rVert_n\le2\mathfrak e\int_0^\infty\lVert\mathbf r\rVert_n\le2H\mathfrak e$.
For the population predictions, subtract the definition of $f^k_t$ from~\eqref{eq:population-chain}.
The cross-kernel discrepancy contributes at most $\frac2n\int_0^t\lVert\sum_i[\widetilde k_{\vartheta(u)}-k](\cdot,x_i)\mathbf r_i(u)\rVert_\rho\,du\le2\mathfrak b\int_0^\infty\lVert\mathbf r\rVert_n\le2H\mathfrak b$, by Cauchy--Schwarz and~\eqref{eq:drift}.
The remaining term is $\frac2n\mathcal S^*\int_0^t(\mathbf r-\mathbf r_k)\,ds$.
The semigroup integral is
\[
\int_u^te^{-2(s-u)K/n}\,ds=\frac n2K^{-1}(I_n-e^{-2(t-u)K/n}).
\]
Integrating~\eqref{eq:duhamel} over the finite time triangle and applying Fubini therefore gives
\[
\frac2n\mathcal S^*\int_0^t(\mathbf r(s)-\mathbf r_k(s))\,ds=-\frac2n\int_0^t\mathcal S^*K^{-1}\bigl(I_n-e^{-2(t-u)K/n}\bigr)\Delta_u\mathbf r(u)\,du .
\]
Here $\Delta_u$ stays to the right of the semigroup.
Using $\lVert\Delta_u\mathbf r(u)\rVert_2\le n^{3/2}\mathfrak e\lVert\mathbf r(u)\rVert_n$ and Lemma~\ref{lem:pop-operator}, we bound this term's population norm by
\[
\frac2n\cdot\frac{C_{\mathrm{pop}}}{\sqrt n}\cdot n^{3/2}\mathfrak e\int_0^\infty\lVert\mathbf r(u)\rVert_n\,du\le2HC_{\mathrm{pop}}\mathfrak e .
\]
On $E_{\mathrm{init}}(\eta)$ both bounds are at most $\eta$.
\end{proof}

\subsection{Proofs of the Main Results}\label{app:proofs-main}

\begin{proof}[Proof of Theorem~\ref{thm:finite}]
We set $\eta=\sqrt\epsilon/4$ and $\delta_0=\delta$, and define
\[
E_{\mathrm{stat}}=\{\lVert\widehat f_{*,\lambda}-f^\star\rVert_\rho\le\sqrt\epsilon/2\}.
\]
Lemma~\ref{lem:kernels}(iii), with the radius in~\eqref{eq:schedule}, and Lemma~\ref{lem:tuned-ridge} give
\[
\Pr[E_{\mathrm{sep}}^c]\le\delta/6,\qquad
\Pr[E_{\mathrm{stat}}^c]\le\delta/3.
\]
Averaging Lemma~\ref{lem:init-event} over separated inputs gives
\[
\Pr[E_{\mathrm{sep}}\setminus E_{\mathrm{init}}(\eta)]\le\delta/3.
\]
The event $E=E_{\mathrm{sep}}\cap E_{\mathrm{init}}(\eta)\cap E_{\mathrm{stat}}$ therefore has probability at least $1-5\delta/6$.
On $E$, let $t\ge T$, so that $e^{-\tau}=e^{-2t/H}\le e^{-2T/H}=\sqrt{\epsilon\lambda}/4\le\sqrt\epsilon/4$, since $\lambda\le1$.
Proposition~\ref{prop:kernel-flow} and Lemma~\ref{lem:coupling} give $\lVert\mathbf r(t)\rVert_n\le\lVert\mathbf r_k(t)\rVert_n+\eta\le\sqrt\epsilon/2$, so $\Rhat(f_{W(t)})\le\epsilon/4$.
The condition $n\ge64/(\epsilon\delta)$ implies $na^q=\delta/(3\cdot2^qn)\le\epsilon/64$.
The triangle inequality, Proposition~\ref{prop:kernel-flow} and Lemma~\ref{lem:coupling} give
\[
\lVert f_{W(t)}-f^\star\rVert_\rho\le\frac{\sqrt\epsilon}4+\Bigl(\frac{\sqrt\epsilon}8+\frac{e^{-\tau}}{2\sqrt\lambda}\Bigr)+\frac{\sqrt\epsilon}2\le\sqrt\epsilon ,
\]
because $e^{-\tau}/(2\sqrt\lambda)\le\sqrt\epsilon/8$.
For the limit, Lemma~\ref{lem:residual} gives convergence and interpolation, $f_{W(t)}\to f_{W(\infty)}$ uniformly and $f^k_t\to f^k_\infty$ in $L^2(\rho)$, so the comparison bound passes to the limit, and the transient term vanishes.
This gives $\lVert f_{W(\infty)}-f^\star\rVert_\rho\le\sqrt\epsilon/4+\sqrt\epsilon/8+\sqrt\epsilon/2=7\sqrt\epsilon/8$.
\end{proof}

\begin{proof}[Proof of Corollary~\ref{cor:log-confidence}]
Lemma~\ref{lem:log-ridge} supplies the same event $E_{\mathrm{stat}}$ as in the proof of Theorem~\ref{thm:finite}, with failure probability at most $\delta/3$.
Its proof also gives $n\ge2r_{\log}=1024L_{\log}/\epsilon$, so $n\ge16$, $4\le r_{\log}\le n$ and $n\ge64/\epsilon$.
The finite-network lemmas therefore apply with the new $r,\lambda,\nu,H,T$ and width, and the unchanged radius gives
\[
na^q=\frac{\delta}{3\cdot2^qn}\le\epsilon/64.
\]
The separation and initialization failures are at most $\delta/6$ and $\delta/3$, respectively, so the common event has probability at least $1-5\delta/6$.
For every $t\ge T$, we have $e^{-2t/H}\le\sqrt{\epsilon\lambda}/4$.
The training comparison and the four population norm budgets in the proof of Theorem~\ref{thm:finite} therefore give empirical risk at most $\epsilon/4$ and excess risk at most $\epsilon$.
Lemma~\ref{lem:residual} gives convergence and exact interpolation; the transient vanishes at the limit, leaving excess risk at most $49\epsilon/64$.
\end{proof}

\begin{proof}[Proof of Theorem~\ref{thm:consistent}]
Let $\Omega_n=E_{\mathrm{sep}}\cap E_{\mathrm{init}}(\eta_n)$ for the parameters of the theorem.
As above, $\Pr[\Omega_n^c]\le\delta_n/6+\delta_n/3\to0$.
On $\Omega_n$ and for $t\ge T_n$, the normalized time satisfies $\tau\ge\log(n)$, so $\lVert\mathbf r(t)\rVert_n\le n^{-1}+\eta_n$ and
\[
\lVert f_{W(t)}-f^\star\rVert_\rho\le\eta_n+\sqrt{na_n^q}+\frac{n^{-1}}{2\sqrt{\lambda_n}}+\lVert\widehat f_{*,\lambda_n}-f^\star\rVert_\rho ,
\]
and the same bound without the third term holds in the limit, which exists and interpolates by Lemma~\ref{lem:residual}.
With $\epsilon_n$ as in Remark~\ref{rmk:rates}, both risks are at most $\epsilon_n$ for all $t\ge T_n$ and in the limit.
Here $\eta_n\to0$, $na_n^q=\delta_n/(3\cdot2^qn)\to0$ and $(4nr_n)^{-1/2}\to0$.
Lemma~\ref{lem:ridge-oracle} with $\ell=k_*$ gives $\E[\lVert\widehat f_{*,\lambda_n}-f^\star\rVert_\rho^2]\le(1+r_n^{-1})^2\mathcal A_*(\lambda_n)+3\sigma^2/r_n\to0$, because $\lambda_n\to0$, $\mathcal A_*(\lambda)\to0$ as $\lambda\downarrow0$ by Lemma~\ref{lem:kernels}(i), and $r_n\to\infty$.
Hence $\epsilon_n\to0$ in probability.
For times $t_n\ge T_n$, $\Pr[\Rhat(f_{W(t_n)})>\epsilon\text{ or }R(f_{W(t_n)})-R(f^\star)>\epsilon]\le\Pr[\Omega_n^c]+\Pr[\epsilon_n>\epsilon]\to0$, and the same holds for the interpolating limit.
\end{proof}

\subsection{Sufficient Network Width}

\rev{Finally, we give an explicit polynomial bound on the number of neurons needed for the desired accuracy and confidence.}
\rev{The cap-mass bound controls $H$; substitution into the width condition then gives explicit powers of $n$, $r$ and $\delta_0^{-1}$.}

\begin{lem}\label{lem:width}
Let $n\ge16$, $4\le r\le n$, $\eta\in(0,1]$, $\delta_0\in(0,1)$ and $a=\frac12(\delta_0/(3n^2))^{1/q}$.
There is a constant $C_d$, depending only on $d$ and given in~\eqref{eq:width-constant}, such that
\[
m_0(n,d,r,a,\eta,\delta_0)\le C_d\eta^{-2}\Bigl[\delta_0^{-3}\frac{n^9}{r^3}+\delta_0^{-(3+4/q)}\frac{n^{6+8/q}}r\Bigr].
\]
\end{lem}

\begin{proof}
By Lemma~\ref{lem:caps}, $v\ge c_{\mathrm{cap}}h^{q/2}=b_d\delta_0n^{-2}$ with $b_d=c_{\mathrm{cap}}/(3\cdot32^{q/2})$.
Since $r\ge4\ge4v$, $\nu=2vr/(r+2v)\ge v$, so $H\le L_d\delta_0^{-1}n^3$ with $L_d=1/b_d$.
Also $D_GH=4\sqrt{2/\pi}n/r$, and, since $v\ge c_{\mathrm{cap}}h^{\omega_d+1}$, $D_LH=4\bar c_dh^{\omega_d}n/v\le K_dn/h$ with $K_d=4\bar c_d/c_{\mathrm{cap}}$.
Moreover $\zeta=(3/8)^{q/2}\delta_0/(3n^2)\le n^{-2}$, $p\le1\le n^2/r^2$ and $h\le1$, so
\[
\mathcal E\le A\frac{n^2}{r^2}+\frac B{nh^2},\qquad \mathcal B\le A\frac{n^2}{r^2}+\frac F{nh^2},
\]
with $A=2048/\pi+144$, $B=32K_d^2+80$ and $F=64K_d^2+144$.
Since $4ra^q\le4n\delta_0/(3\cdot2^qn^2)<1$, we have $C_{\mathrm{pop}}^2\le\frac n{2r}+2na^q\le\frac nr$, so $(1+C_{\mathrm{pop}}^2)\mathcal E+\mathcal B\le3An^3/r^3+(2B+F)/(rh^2)$.
With $h^{-2}=1024\cdot3^{4/q}\delta_0^{-4/q}n^{8/q}$, the second term in~\eqref{eq:width} is at most
\[
72AL_d^2\eta^{-2}\delta_0^{-3}\frac{n^9}{r^3}+24576\cdot3^{4/q}(2B+F)L_d^2\eta^{-2}\delta_0^{-(3+4/q)}\frac{n^{6+8/q}}r,
\]
and the first term, $64H^2/(nh^2)\le65536\cdot3^{4/q}L_d^2\delta_0^{-(2+4/q)}n^{5+8/q}$, is smaller than the last display because $\eta\le1$, $\delta_0<1$, $r\le n$ and $2B+F\ge3$.
Rounding up to an even integer adds at most $2$, and the bracket in the claim is at least one, so the claim holds with
\begin{equation}\label{eq:width-constant}
C_d=2+72AL_d^2+24576\cdot3^{4/q}(2B+F)L_d^2 .
\end{equation}
For the schedule~\eqref{eq:schedule}, with $\eta=\sqrt\epsilon/4$ and $\delta_0=\delta$, the bracket times $\eta^{-2}$ equals
\[
\frac{16\epsilon^2n^9}{64^3}+\frac{\delta^{-(2+4/q)}n^{6+8/q}}4,
\]
which gives the bound stated in Theorem~\ref{thm:finite}.
\end{proof}

\section{\rev{Numerical Details}}\label{app:experiments}

\rev{We describe the numerical evaluation of the localized kernel flow in Section~\ref{sec:experiments}.}

\paragraph{Funk--Hecke Computations.}
\rev{Write the global kernel on $\mathbb S^2$ as $k_*(x,u)=\kappa(x^\top u)$.}
\rev{Its expansion is $\kappa(t)=\sum_{l\ge0}(2l+1)\mu_lP_l(t)$} with Legendre polynomials $P_l$ and $\mu_l=\frac12\int_{-1}^1\kappa(t)P_l(t)\,dt$, and the Funk--Hecke formula gives
\[
\int\kappa(x^\top z)\kappa(z^\top u)\,d\rho(z)=\sum_{l\ge0}(2l+1)\mu_l^2P_l(x^\top u),\qquad \int\kappa(x^\top z)z_3\,d\rho(z)=\mu_1x_3 .
\]
We compute $\mu_l$ for $l\le400$ by quadrature on $40001$ equally spaced angles; the trace check $\sum_l(2l+1)\mu_l=\kappa(1)$ holds to within $6\cdot10^{-4}$.
The excess risk of a predictor $\sum_i\alpha_i\rev{k_*(\cdot,x_i)}$ is $\alpha^\top M\alpha-2\mu_1\alpha^\top f^\star(X)+\lVert f^\star\rVert_\rho^2$, with $M_{ij}$ the first expression above at $x_i^\top x_j$, interpolated from the angular grid.

\paragraph{Global Kernel Evaluation.}
For $x^\top u=\cos(\theta)$, the vector $Z$ enters $k_*$ only through its projection onto \rev{a two-dimensional plane containing $x$ and $u$}, which is a standard Gaussian vector.
In polar coordinates $(R,\varphi)$, both $Z^\top x>1$ and $Z^\top u>1$ hold on an arc of angular length $\max\{0,2\arccos(1/R)-\theta\}$, so
\[
\Pr[Z^\top x>1,\ Z^\top u>1]=\frac1{2\pi}\int_{1/\cos(\theta/2)}^\infty\bigl(2\arccos(1/R)-\theta\bigr)Re^{-R^2/2}\,dR\rev{.}
\]
\rev{At $\theta=\pi$ the probability is zero; for $0\le\theta<\pi$, we evaluate the integral by adaptive quadrature on every tenth angle of the grid and interpolate.}
The trace check gives $\sum_l(2l+1)\mu_l(k_*)=0.3168$, against $a_*=0.3173$.

\paragraph{Local Kernel Integrals.}
On separated samples the localized kernel flow is
\[
f^k_t=r^{-1}k_*(\cdot,X)^\top\beta+\sum_i\beta_i\psi_{x_i},\qquad \beta=r\alpha_t=(I_n+G/r)^{-1}\bigl(I_n-e^{-\tau(I_n+G/r)}\bigr)y,
\]
where $\psi_u=k_h(\cdot,u)/(2v)$.
Define $q_\psi=\lVert\psi_u\rVert_\rho^2$ and $\mu_\psi=\int\psi_u\,d\rho$, which are independent of $u$ by rotation invariance.
Disjoint supports give $\lVert\sum_i\beta_i\psi_{x_i}\rVert_\rho^2=q_\psi\lVert\beta\rVert_2^2$.
For continuous $g$, we approximate \rev{$\int\psi_{x_i}g\,d\rho$} by $\mu_\psi g(x_i)$.

If $\omega_g(a)=\sup_{\lVert x-u\rVert\le a}\{|g(x)-g(u)|\}$, the absolute error of this substitution is at most $\mu_\psi\omega_g(a)$, because $\psi_{x_i}\ge0$ is supported within distance $a$ of $x_i$.
For a Lipschitz function this is at most \rev{$\mu_\psi a$ times its Lipschitz constant}.
To compute $q_\psi$ and $\mu_\psi$, we use a planar approximation, appropriate for activation radii between $1.3\cdot10^{-5}$ and $2.5\cdot10^{-4}$, which replaces $\psi_u(x)$ by the normalized intersection area of two discs of radius $a/2$ whose centers are at distance $\lVert x-u\rVert$; the numerical quadratic forms use these approximate integrals.
The training risk is $\lVert e^{-\tau(I_n+G/r)}y\rVert_n^2$.

\paragraph{\rev{Time Grid and Risk Estimates.}}
\rev{We use $600$ logarithmically spaced normalized times in $[10^{-3},10^3]$.
Figure~\ref{fig:main}(b) reports excess risk at the final time $\tau=10^3$.
For each sample size, we use five samples and retain inputs satisfying $E_{\mathrm{sep}}$; the seed identifiers are the same across sample sizes.
The script \nolinkurl{kernel_flow_curves.py} generates the saved numerical results, and \nolinkurl{make_localized_figure.py} reproduces Figure~\ref{fig:main} from their localized-kernel entries.}

\bibliographystyle{emss}
\bibliography{ref}

\end{document}